%% file: fkt.tex
\documentclass{article}

\usepackage{times}
\usepackage{graphicx} 
\usepackage{subfigure} 
\usepackage[table]{xcolor}
\usepackage{multirow}
\usepackage{natbib}
\usepackage{paralist}
\usepackage{amsmath}
\usepackage{amsthm}
\usepackage{amssymb}
\usepackage{amsfonts}

\usepackage{algorithm}
\usepackage{algorithmic}

\PassOptionsToPackage{usenames,dvipsnames,svgnames}{xcolor}  
\usepackage{graphicx}
\usepackage{tikz}
\usetikzlibrary{arrows, decorations.markings, positioning, automata}
\usetikzlibrary{fit}					
\usetikzlibrary{backgrounds}
\usepackage{tkz-graph}

\usepackage{hyperref}

\usepackage[accepted]{icml2016}

\newtheorem{theorem}{Theorem}[section]
\newtheorem{theorem_k}{Theorem}[section]
\newtheorem{definition}{Definition}[section]
\newtheorem{lemma}{Lemma}[section]

\input{notation.tex}

\tikzstyle{input3}=[circle,
                                    thick,
                                    minimum size=1.0cm,
                                    draw=green!80,
                                    fill=white!20]
                                    
\tikzstyle{input4}=[circle,
                                    ultra thick,
                                    minimum size=1.0cm,
                                    draw=black!80,
                                    fill=red!20]

\tikzstyle{input}=[circle,
                                    thick,
                                    minimum size=2.2cm,
                                    draw=black!80,
                                    fill=white!20]

\tikzstyle{input2}=[circle,
                                    thick,
                                    minimum size=1.0cm,
                                    draw=black!80,
                                    fill=white!20]

\tikzstyle{matrx}=[rectangle,
                                    thick,
                                    minimum size=1.8cm,
                                    draw=black!80,
                                    fill=white!20]

\tikzstyle{matrx2}=[rectangle,
                                    thick,
                                    minimum size=1.5cm,
                                    draw=black!80,
                                    fill=gray!20]

\tikzstyle{vecArrow} = [thick, decoration={markings,mark=at position
   1 with {\arrow[semithick]{open triangle 60}}},
   double distance=1.8pt, shorten >= 5.5pt,
   preaction = {decorate},
   postaction = {draw,line width=1.4pt, white,shorten >= 4.5pt}]
\tikzstyle{innerWhite} = [semithick, white,line width=1.4pt, shorten >= 4.5pt]

\tikzstyle{background}=[rectangle,
                                                fill=gray!10,
                                                inner sep=0.2cm,
                                                rounded corners=5mm]

\icmltitlerunning{Recycling Randomness with Structured Matrices}

\icmltitlerunning{Recycling Randomness with Structured Matrices}

\begin{document} 

\twocolumn[
\icmltitle{Recycling Randomness with Structure for Sublinear time Kernel Expansions}

\icmlauthor{Krzysztof Choromanski}{kchoro@google.com}
\icmlauthor{Vikas Sindhwani}{sindhwani@google.com}


\vskip 0.3in
]



\begin{abstract}
\input{abstract.tex}
\end{abstract} 
\def\K{\cal K}
\section{Introduction}
\label{sec:intro}
\input{introduction.tex}

\input{theory.tex}

\input{experiments.tex}

\section{Conclusions}
\label{sec:conclusions}
We have theoretically justified and empirically validated the use of a broad family of structured matrices 
for accelerating the construction of random embeddings for approximating various kernel functions. 
In particular, the class of Toeplitz-like semi-Gaussian matrices allows our construction to span highly compact to fully random matrices. 



\bibliography{nonlinear_embed_ldrm,qmc,hilbert,RandomLaplace,structured}
\bibliographystyle{icml2016}
\newpage
\input{appendix.tex}

\end{document}


%% file: notation.tex
\def\P{\mathcal{P}}

\def\reals{\mathbb{R}}

\def\Expectation{\mathbb{E}}

\newcommand{\vv}[1] {\mathbf{#1}}

\def\x{\vv{x}}

\def\w{\vv{w}}
\def\z{\vv{z}}

%% file: abstract.tex
We propose a scheme for recycling Gaussian random vectors into structured matrices to approximate various kernel functions  in sublinear time via random embeddings.  Our framework includes the Fastfood construction of~\citet{LSS13}  as a special case, but also extends to Circulant, Toeplitz and Hankel matrices, and the broader family of structured matrices that are characterized by the concept of low-displacement rank.  We introduce notions of  coherence and graph-theoretic structural constants that control the approximation quality, and prove unbiasedness and low-variance properties of random feature maps that arise within our framework. For the case of low-displacement matrices, we show how the degree of structure and randomness can be controlled to reduce statistical variance at the cost of increased computation and storage requirements. Empirical results strongly support our theory and justify the use of a  broader family of structured matrices for scaling up kernel methods using random features.

%% file: introduction.tex
\def\K{{\cal K}}
Consider a $k$-dimensional feature map of the form,
\begin{equation}
\Psi(\vv{x}) = \frac{1}{\sqrt{k}} s(\vv{M} \vv{x}) \label{eq:featuremap}
\end{equation} where the input data vector $\vv{x}$ is drawn from $\reals^n$,  $s(\cdot)$ denotes a real-valued or complex-valued pointwise non-linearity (activation function), and $\vv{M}$ is a $k\times n$ Gaussian random matrix. It is well known that  as a function of a pair of data vectors,  the Euclidean inner product $\Psi(\vv{x})^T \Psi(\vv{z})$, converges to a positive definite kernel function $\K(\vv{x}, \vv{z})$  depending on the choice of the scalar nonlinearity, as $k\to\infty$. For example, the complex exponential nonlinearity $s(x) = e^{-i \frac{x}{\sigma}}$ corresponds to the Gaussian kernel~\cite{RR07}, while the rectified linear function (ReLU), $s(x) = \max(x, 0)$, leads to the Arc-cosine kernel~\cite{cho}.

In recent years, such random feature maps have been used to dramatically accelerate the training time and inference speed of kernel methods~\cite{KernelsBook} across a variety of statistical modeling problems~\cite{RR07, LeSong} and applications~\cite{TIMIT, VedaldiZisserman2012}. Standard linear techniques applied to random nonlinear embeddings of data are equivalent to learning with approximate kernels. To quantify the benefits, consider solving a kernel ridge regression task given $l$ training examples. With traditional kernel methods, dense linear algebra operations on the Gram matrix associated with the exact kernel function imply that the training complexity grows as $O(l^3 + l^2n)$ and the time to make a prediction on a test sample grows as $O(ln)$. By contrast, random feature approximations reduce training complexity to $O(l k^2 + lkn)$ and test speed to $O(kn)$. This is a major win on big datasets where $l$ is very large, provided that a small value of $k$ can provide a good approximation to the kernel function. 

In practice, though, the optimal value of $k$ is often large, albeit still much smaller than $l$. For example, in a speech recognition application~\cite{TIMIT} involving around two million training examples, about hundred thousand random features are required to achieve state of the art results. In such settings, the time to construct the random feature map is dominated by matrix multiplication against the dense Gaussian random matrix, which becomes the new computational bottleneck. To alleviate this bottleneck,~\cite{LSS13} introduce the ``Fastfood" approach where Gaussian random matrices are replaced by Hadamard matrices combined with diagonal matrices with Gaussian distributed diagonal entries. It was shown in~\cite{LSS13} that for the specific case of the complex exponential nonlinearity, the Fastfood feature maps provide unbiased estimates for the Gaussian kernel function, at the expense of additional statistical variance, but with the computational benefit of reducing the feature map construction time from $O(k n)$ to $O(k \log~n)$ by using the Fast Walsh-Hadamard transform for matrix multiplication.  The Fastfood construction for kernel approximations is akin to the use of structured matrices - in lieu of Gaussian random matrices - in Fast Johnson-Lindenstrauss transform (FJLT)~\cite{FJLT} for dimensionality reduction, fast compressed sensing~\cite{ToeplitzCS,CirculantCS}, and randomized numerical linear algebra techniques~\cite{HMT,Mahoney}
Specific structured matrices were recently applied for approximating angular kernels~\cite{choromanska2015binary}. 
Some heuristic results for approximating kernels with circulant matrices were given in~\cite{felix}. 

Our contributions in this paper are as follows:
\setlength{\itemindent}{0em}
\begin{compactitem}
\item We study a general family of structured random matrices that can be constructed by recycling a Gaussian random vector using a sequence 
of elementary generator matrices (introduced in Section 3). This family includes Circulant, Toeplitz and Hankel matrices. 
It also includes the Fastfood construction of~\cite{LSS13} as a special case. We show that fast sublinear time random feature 
maps obtained from these matrices provide unbiased estimates of the exact kernel, with variance comparable to the fully unstructured Gaussian case (Section 4). 
We introduce various structural coherence and graph-theoretic constants that control the quality of randomness we get from our model. 
Our approach generalizes across various choices of nonlinearities and kernel functions. 
\item Of particular interest for us is the class of generalized structured matrices that have low-displacement rank~\cite{Pan,StructuredTransforms}. Such matrices span an increasingly rich class of structures as the displacement rank is increased: from Circulant and Toeplitz matrices, to inverses and products of Toeplitz matrices, and more. The displacement rank provides a knob with which the degree of structure and randomness can be controlled to tradeoff computational and storage requirements against statistical variance. 
\item We provide empirical support for our theoretical results (Section 5). In particular, we show that Circulant, Fastfood and low-displacement 
Toeplitz-like matrices provide high quality sublinear-time feature maps for approximating various kernels. 
With increasing displacement rank, the quality of the approximation approaches that of the fully Gaussian random matrix.
\end{compactitem}


\section{Background and Preliminaries}
We start by giving a brisk background on random feature maps and structured matrices.

\subsection{Random Embeddings, Nonlinearities and Kernels}

Random feature maps may be viewed as arising from Monte-Carlo approximations to integral representations of kernel functions. The original construction by ~\citet{RR07} was motivated by a classical result that characterizes the class of shift-invariant positive definite functions.

\begin{theorem}[Bochner's Theorem~\cite{bochner1933}] A continuous shift-invariant scaled kernel
function $\K(\x,\z)\equiv \phi(\x-\z)$ on $\reals^n$ is positive definite if and
only if it is the Fourier transform of a unique finite probability measure $p$ on
$\reals^n$.
That is, for any $\x,\z \in \reals^d $,
\begin{equation}
\K(\x,\z) = \int_{\reals^n} e^{-i (\x-\z)^T\vv{w}} p(\w) d\w =
\Expectation_{\vv{w}\sim p}[ e^{-i (\x-\z)^T\vv{w}}]~.\nonumber\label{eq:bochner}
\end{equation}
\end{theorem} \label{thm:bochner_R_d}
Bochner's theorem stablishes one-to-one correspondence between shift-invariant kernel functions and probability densities on $\reals^n$, via the Fourier transform.  In the case of the Gaussian kernel with bandwidth $\sigma$, the associated density is also Gaussian with covariance matrix $\sigma^{-2}$ times the identity. 

While studying synergies between kernel methods and deep learning,~\cite{cho} introduce $b^{th}$-order {\it arc-cosine} kernels via the following integral representation:
$$
{\K}_{b}(\vv{x}, \vv{z}) = \int_{\reals^d} i(\vv{w}^T \vv{x})  i(\vv{w}^T \vv{z}) (\vv{w}^T\vv{x})^b (\vv{w}^T \vv{z})^b \ p(\w) d\w 
$$ 
where $i(\cdot)$ is the step function, i.e.  $i(x) = 1$ if $x>0$ and $0$ otherwise; and the density $p$ is chosen to be standard Gaussian. 
These kernels evaluate inner products in the  representation induced by an infinitely wide single  hidden layer neural network with random 
Gaussian weights, and admit closed form expressions in terms of the angle $\theta = cos^{-1} (\frac{\vv{x}^T\vv{z}}{\|\vv{x}\|_2\|\vv{z}\|_2})$ between $\vv{x}$ and $\vv{z}$:
\begin{eqnarray}
\K_{0}(\vv{x}, \vv{z}) &=& 1 - \frac{\theta}{\pi}\label{eq:stepkernel}\\  
\K_{1}(\vv{x}, \vv{z}) &=&  \frac{\|\vv{x}\|_2 \|\vv{z}\|_2}{\pi} [sin(\theta) + (\pi - \theta) cos(\theta)] \label{eq:relukernel}
\end{eqnarray} where $\| \cdot\|_2$ denotes $l_2$ norm. 

Monte Carlo approximations to the integral representations above lead to the following,
\begin{equation}
\K(\x, \z) \approx \frac{1}{k}\sum_{j=1}^k s(\x^T \w_j) s(\z^T \w_j) =
\Psi(\x)^T \Psi(\z) 
\end{equation} where the feature map $\Psi(\x)$ has the form given in Eqn.~\ref{eq:featuremap},  with rows of $\vv{M}$, 
i.e. $\vv{w}_j$ vectors, drawn from the Gaussian density, and the nonlinearity $s$ set to the following:  complex exponential, $s(x) = e^{i \frac{x}{\sigma}}$, for the Gaussian kernel with bandwidth $\sigma$;  hard-thresholding, $s(x) = i(x)$, for the angular similarity kernel in Eqn.~\ref{eq:stepkernel}; and  ReLU activation,  $s(x) = \max(x, 0)$, for the first order arc-cosine kernel in Eqn.~\ref{eq:relukernel}.

\subsection{Structured Matrices}

A $m\times n$ matrix is called a structured matrix if it satisfies the following two properties: (1) it has much fewer degrees of freedom than 
$mn$ independent entries, and hence can be implicitly stored more efficiently than general matrices, and (2) the structure in the matrix can 
be exploited for fast linear algebra operations such as fast matrix-vector multiplication.  Examples include the Discrete Fourier Transform (DFT), the Discrete Cosine Transform (DCT) and the Walsh-Hadamard Transform (WHT) matrices.  Here, we give other examples particularly relevant to this paper. The matrices described below are square. Rectangular matrices can be obtained by appropriately selecting rows or columns.

\def\krylov{\textrm{krylov}}
{\bf Circulant Matrices}: These matrices are intimately associated with circular  
convolutions and have been used for fast compressed sensing in~\cite{CirculantCS}. 
A $n\times n$ Circulant matrix is completely determined by  its first column/row, i.e., $n$ parameters. Each column/row of a Circulant 
matrix is generated by cyclically down/right-shifting the previous column/row. A {\it skew-Circulant} matrix has identical structure to Circulant, except that the upper triangular part of the matrix is negated.  This general structure looks like, 
{\small
\begin{eqnarray}
\left[\begin{array}{cccc} 
{\bf \color{red}{g_0}} & f {\bf g_{n-1}} & \ldots & f {\bf \color{blue} g_1}\\
{\bf \color{blue} g_1}& {\bf \color{red} g_0} & \ldots & \vdots\\
\vdots & \vdots & \vdots & f{\bf g_{n-1}} \\
{\bf g_{n-1}} & \ldots & {\bf \color{blue} g_1} & {\bf \color{red}g_0}
 \end{array}
 \right] \nonumber
\end{eqnarray}}
with $f=1$ for Circulant and $f=-1$ for skew-Circulant matrix.  Both these matrices admit $O(n~\log~n)$ matrix-vector multiplication as they are diagonalized by the DFT matrix~\cite{Pan}. We will use the notation $\texttt{circ}[\vv{g}]$ and $\texttt{scirc}[\vv{g}]$ for Circulant and skew-Circulant matrices respectively.

{\bf Toeplitz and Hankel Matrices:} These matrices implement discrete linear convolution and arise naturally in dynamical systems and time series analysis. Toeplitz matrices are characterized by constant diagonals as follows,
{\small \begin{eqnarray}
\left[\begin{array}{cccc} 
{\bf \color{red}{t_0}} & {\bf \color{green} t_{-1}} & \ldots & {\bf \color{cyan}{t_{-(n-1)}}}\\
{\bf \color{blue} t_1}& {\bf \color{red} t_0} & \ldots & \vdots\\
\vdots & \vdots & \vdots & {\bf \color{green} t_{-1}} \\
{\bf t_{n-1}} & \ldots & {\bf \color{blue} t_1} & {\bf \color{red}t_0}
 \end{array}
 \right] \nonumber
\end{eqnarray}} Closely related Hankel matrices have constant anti-diagonals. Toeplitz-vector multiplication can be reduced to $O(n~\log~n)$ Circulant-vector multiplication. For detailed properties of Circulant and Toeplitz matrices, we point the reader to~\cite{GrayCirculantToeplitz}

{\bf Structured Matrices with Low-displacement Rank}: The notion of displacement operators and displacement rank~\cite{Golub, Pan, KKM79} can be used to broadly generalize various classes of structured matrices.  For example, under the action of the {\it Sylvester displacement operator} defined as $L[\vv{T}] = \vv{Z}_1 \vv{T} - \vv{T} \vv{Z}_{-1}$,  every Toeplitz matrix can be transformed into a matrix of rank at most $2$ using elementary shift and scale operations implemented by matrices of the form $\vv{Z}_f = [\vv{e}_2 \vv{e}_3\ldots \vv{e}_n~f \vv{e}_1]$ for $f=\pm 1$ where $\vv{e}_1\ldots \vv{e}_n$ are column vectors representing the standard basis of $\reals^n$. 

For a given displacement rank parameter $r$, the class of matrices for which the rank of $L[\vv{T}]$ is at most $r$ is called {\it Toeplitz-like}. Remarkably, this class of matrices admits a closed-form parameterization in terms of the low-rank factorization of $L[\vv{T}]$:

\begin{theorem}[Parameterization of Toeplitz-like matrices with displacement rank $r$~\cite{Pan}]:  If an $n\times n$ matrix $\vv{T}$ satisfies $rank(\vv{Z}_1 \vv{T} - \vv{T} \vv{Z}_{-1}) \leq r$,  then it can be written as,
\begin{equation}
\vv{T} = \sum_{i=1}^r \texttt{circ}[\vv{g}^i]~\texttt{scirc}[\vv{h}^i] \label{eq:circ_skewcirc}
\end{equation} for some choice of vectors $\{\vv{g}^i, \vv{h}^i\}_{i=1}^r \in \reals^n$.
\end{theorem}   The family of matrices expressible by Eqn.~\ref{eq:circ_skewcirc} is very rich~\cite{Pan}, i.e., it covers (i) 
all Circulant and Skew-circulant matrices for $r=1$, (ii) all Toeplitz matrices and their inverses for $r=2$, (iii) Products, inverses, linear combinations of distinct 
Toeplitz matrices with increasing $r$, and (iv) all $n\times n$ matrices for $r=n$. 
Since Toeplitz-like matrices under the parameterization of Eqn.~\ref{eq:circ_skewcirc} 
are a sum of products between Circulant and Skew-circulant matrices, they inherit fast FFT based matrix-vector multiplication with cost $O(nr log~n)$,
where $r$ is the displacement rank. Hence, $r$ provides a knob on the degree of structure imposed on the matrix with which storage requirements,  computational constraints and statistical capacity can be explicitly controlled. Recently such matrices were used in the context of  learning mobile-friendly neural networks in~\cite{StructuredTransforms}. 
We note in passing that the displacement rank framework generalizes to other types of base structures (e.g. Vandermonde); see~\cite{Pan}.
 



\subsection{FastFood}
In the context of fast kernel approximations,~\cite{LSS13} introduce the Fastfood  technique where the matrix $\vv{M}$ in Eqn.~\ref{eq:featuremap} is parameterized by a 
product of diagonal and simple matrices as follows:
\begin{equation}
\vv{F} = \frac{1}{\sqrt{n}}\vv{S}\vv{H}\vv{G} \vv{P} \vv{H} \vv{B}.\label{eq:fastfood}
\end{equation} Here, $\vv{S}, \vv{G}, \vv{B}$ are diagonal random matrices, $\vv{P}$ is a permutation matrix and $\vv{H}$ is the Walsh-Hadamard matrix.  
The $k\times n$ matrix $\vv{M}$ is obtained by vertically stacking $k/n$ independent copies of the $n\times n$ matrix $\vv{F}$. Multiplication  against such a matrix can be performed in time $O(k\log~n)$. The authors prove that (1) the Fastfood approximation is 
unbiased, (2) its variance is at most the variance of standard Gaussian random features with an additional $O(\frac{1}{k})$ 
term, and (3) for a given error probability $\delta$, the pointwise approximation error of a $n\times n$ block of Fastfood is 
at most $O(\sqrt{\log(n/\delta)})$ larger than that of standard Gaussian random features. However, note that the Fastfood analysis 
is limited to the Gaussian kernel and their variance bound uses properties of the complex exponential. The authors also conjecture 
that the Hadamard matrix $\vv{H}$ above, can be replaced by any matrix $\vv{T}$ such that $\vv{T}/\sqrt{n}$ is orthonormal, the 
maximum entry in $\vv{T}$ is small, and matrix-vector product against $\vv{T}$ can be computed in $O(n\log~n)$ time.

%% file: theory.tex
\section{Structured Matrices from Gaussian Vectors}
In this section, we present a general structured matrix model that allows a small Gaussian vector to be recycled in order to mimic the properties of a Gaussian random matrix suitable for generating random features. We first introduce some basic concepts in our construction. Note that we emphasize intuitions in our exposition - formal proofs are provided in our supplementary material.

\subsection{The $\mathcal{P}$-model}
\label{sec:p_model}

{\bf Budget of Randomness}:  Let $t$ be some given parameter. Consider the column vector $\vv{g}=(g_{1},...,g_{t})^T$, where each entry is an independent Gaussian taken from $\mathcal{N}(0,1)$. This vector stands for the ``budget of randomness'' used in our structured matrix construction scheme.

Our goal is to recycle the Gaussian vector $\vv{g}$ to construct random matrices with desirable properties. This is accomplished using a sequence of matrices which we call the $\P$-model.

\begin{definition}[$\P$-model] 
\label{def:p_model}
Given the budget of uncertainty parameter $t$, a sequence  of 
$m$ matrices with unit $l_{2}$ norm columns, denoted as $\P = \{\vv{P}_i\}_{i=1}^m$, where $\vv{P}_i \in \reals^{t \times n}$,  specifies a $\P$-model. 
Such a sequence defines an $m\times n$ random matrix of the form:
 \begin{equation}
 \vv{S}[\P] = \left(\begin{array}{c}\vv{g}^T\vv{P}_1 \\\vv{g}^T\vv{P}_2\\ \vdots\\\vv{g}^T\vv{P}_m\end{array}\right)\label{eq:pmodel}
 \end{equation} where $\vv{g}$ is a Gaussian random vector of length $t$.
\end{definition}

In the constructions of interest to us, the sequence $\P$ is designed to separate structure from Gaussian randomness; 
though elements of $\P$ can be deterministic or itself random,  Gaussianity is restricted to the vector $\vv{g}$.  
The ability of $\P$ to recycle a Gaussian vector effectively depends on certain structural constants that we now define.

\begin{definition}[Coherence of a $\P$-model] For $\P = \{ \vv{P}_i \}_{i=1}^m$, let $\vv{P}_{ij}$ denote the $j^{th}$ 
column of the $i^{th}$ matrix. The coherence of a $\P$-model is defined as,
\begin{equation}
\mu[\P] = \max_{1 \leq i \leq j \leq m}\sqrt{\frac{\sum_{1 \leq n_{1} < n_{2} \leq n} (\vv{P}^{T}_{i,n_{1}} \vv{P}_{j,n_{2}})^{2}}{n}}
\end{equation}
\end{definition}
Note that $\mu[\P]$ is a maximum over all pairs of rows $1 \leq i \leq j \leq m$ of the rescaled sums of cross-correlations $\vv{P}^{T}_{i,n_{1}} \vv{P}_{j,n_{2}}$ 
for all pairs of different column indices $n_{1},n_{2}$.
Lower values of $\mu[\P]$ will lead to better quality models. In practice, as we will see in subsequent analysis, it suffices if $\mu[\P] = O(poly(\log(n)))$ which is the case for instance for Toeplitz and Circulant matrices.

The coherence of the $\mathcal{P}$-model is an extremal statistic of pairwise correlations.  We couple it with another set of objects describing global structural properties of the model, namely the \textit{coherence graphs}.
\begin{definition}[Coherence Graphs for $\P$-model and their Chromatic Numbers] 
Let $1 \leq i,j \leq m$. We define by $\mathcal{G}_{i,j}$ an undirected graph
with the set of vertices $V(\mathcal{G}_{i,j}) = \{\{n_{1},n_{2}\} : 1 \leq n_{1} \neq n_{2} 
\leq n$ and $\vv{P}^{T}_{i,n_{1}}\vv{P}_{j,n_{2}} \neq 0\}$ and the set of edges 
$E(\mathcal{G}_{i,j}) = \{\{\{n_{1},n_{2}\},\{n_{2},n_{3}\}\}:\{n_{1},n_{2}\},\{n_{2},n_{3}\} 
\in V(\mathcal{G}_{i,j})\}$. In other words, edges are between these vertices such that their corresponding $2$-element subsets intersect.
The chromatic number $\chi(i, j)$ of a graph $\mathcal{G}_{i,j}$ is the smallest number of colors that can be used to color all vertices
of $\mathcal{G}_{i,j}$ in such a way that no two adjacent vertices share the same color.
\end{definition}

The chromatic number of a $\mathcal{P}$-model is defined as follows:

\begin{definition}[Chromatic number of a $\mathcal{P}$-model]
The chromatic number $\chi[\mathcal{P}]$ of a $\mathcal{P}$-model is given as:
$$\chi[\mathcal{P}] = \max_{1 \leq i \leq j \leq m} \chi(i, j),$$
where $\mathcal{G}_{i,j}$ are associated coherence graphs.
\end{definition}
As it was the case for the coherence $\mu[\mathcal{P}]$, smaller values of the chromatic number $\chi[\mathcal{P}]$ lead to better theoretical results regarding 
the quality of the model.
Intuitively speaking, coherence graphs encode in a compact combinatorial way correlations between different rows of the structured matrix produced 
by the $\mathcal{P}$-model.
The chromatic number $\chi[\mathcal{P}]$ is a single combinatorial parameter measuring quantitatively these dependencies. It can be easily computed
or at least upper-bounded (which is enough for us) for $\mathcal{P}$-models related to all structured matrices considered in this paper.
The following is a well-known fact from graph theory:
\begin{lemma}
\label{graph_lemma}
The chromatic number $\chi(G)$ of an undirected graph $G$ with maximum degree $d_{max}$ satisfies: $\chi(G) \leq d_{max} + 1$. 
\end{lemma}
For all instantiations of $\mathcal{P}$-models considered in this paper leading to various structured matrices,  the vertices of associated coherence graphs will turn out to have small degrees and hence, by Lemma \ref{graph_lemma}, small chromatic numbers.


We will introduce one more structural parameter of the $\mathcal{P}$-model, depending on whether it is specified deterministically or randomly.
\begin{definition}
The uni-coherence $\tilde{\mu}[\mathcal{P}]$ of the $\mathcal{P}$-model is defined as follows. 
If matrices $\vv{P}_{i}$ are constructed deterministically then
$\tilde{\mu}[\mathcal{P}] = \max_{1 \leq i < j \leq m} \sum_{n_{1}=1}^{n}|\vv{P}^{T}_{i,n_{1}}\vv{P}_{j,n_{1}}|.$
If the matrices that specify $\mathcal{P}$ are constructed randomly, then we take
$\tilde{\mu}[\mathcal{P}] = \max_{1 \leq i < j \leq m} \mathbb{E}[|\sum_{n_{1}=1}^{n}\vv{P}^{T}_{i,n_{1}}\vv{P}_{j,n_{1}}|].$
\end{definition}

It turns out that the sublinearity in $n$ of uni-coherence $\tilde{\mu}[\mathcal{P}]$ helps to establish strong theoretical results
regarding the quality of the $\mathcal{P}$-model.

\subsection{Examples of $\mathcal{P}$-model structured matrices}
\label{sec:examples}
Below we observe that various structured random matrices can be constructed according to the $\mathcal{P}$-model, i.e. by specifying a sequence of matrices $\vv{P}_i$ in Eqn.~\ref{eq:pmodel}. We note that chromatic numbers and coherence values of these $\P$-models are low. In the next section, we show that this implies that we can get unbiased, low-variance kernel approximations from these matrices, for various choices of nonlinearities. Here we consider square structured matrices for which $m = n$, or rectangular matrices with $m < n$ obtained by selecting first $m$ rows of a structured matrix. 

\subsubsection{Circulant matrices}
Circulant matrices can be constructed via the $\mathcal{P}$-model with budget of randomness $t=n$
and matrices $\{\vv{P}_{i}\}_{i=1}^m$ of entries in $\{0,1\}$.  See Fig. \ref{fig:cycle} for an illustrative construction.
The coherence of the related $\mathcal{P}$-model trivially satisfies: $\mu[\mathcal{P}] = O(1)$ and $\tilde{\mu}[\mathcal{P}] = 0$.
The coherence graphs are vertex disjoint cycles. Since each cycle can be colored with at most $3$ colors, the chromatic number of the 
$\mathcal{P}$-model satisfies: $\chi[\mathcal{P}] \leq 3$.

\begin{figure}[h]
\vspace{-0.05in}
\vspace{-0.1in}
\centering
\includegraphics[height=6cm, width = 3.35in]{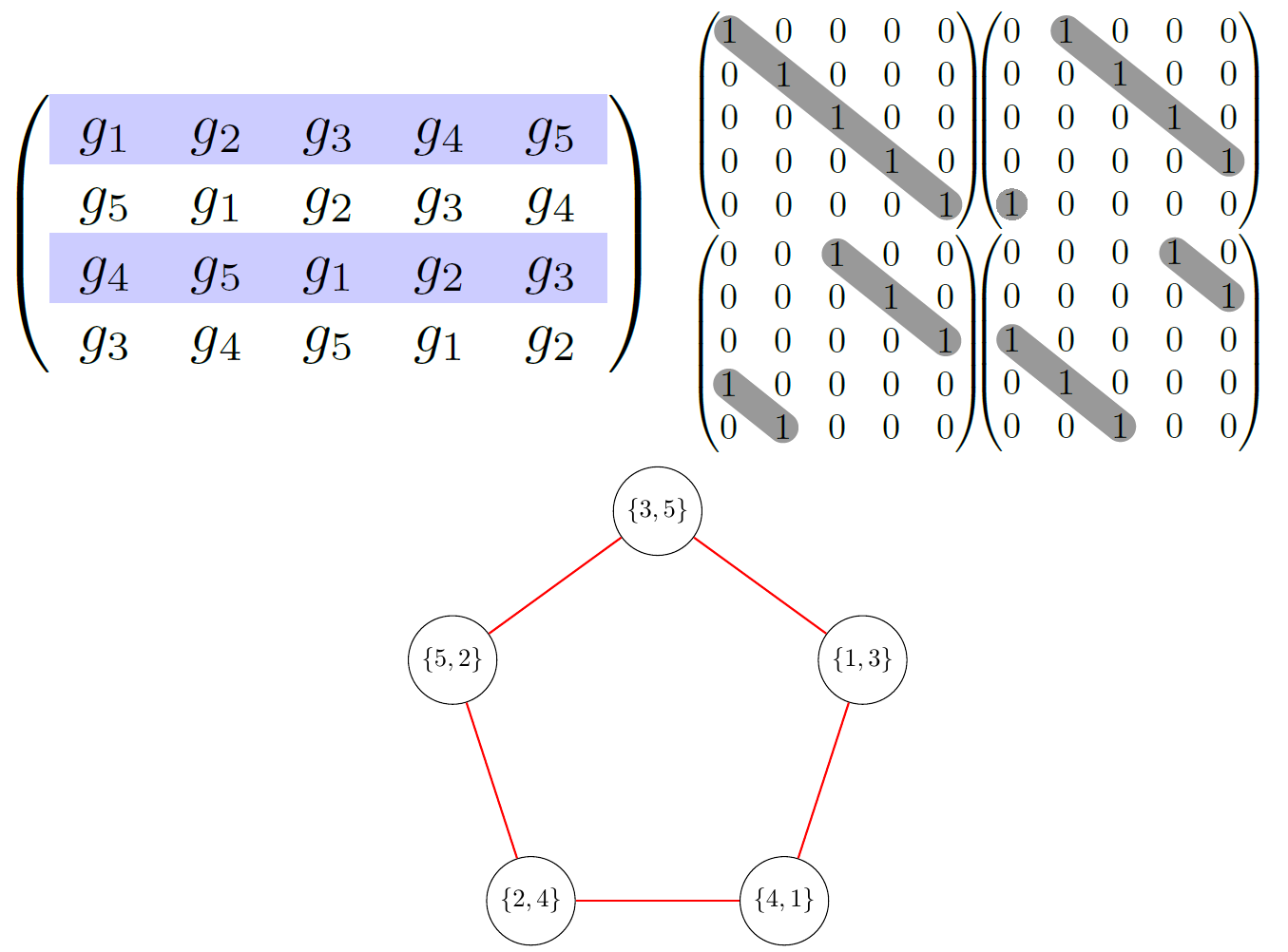}
\vspace{-0.1in}
\caption{Top left: Circulant gaussian matrix $\mathcal{C}$. Top right: matrices $\vv{P}_{1},\vv{P}_{2},\vv{P}_{3},\vv{P}_{4}$ from
the $\mathcal{P}$-model generating $\mathcal{C}$ from the ``budget of randomness'' $(g_{1},...,g_{5})$.
Bottom: Graph $\mathcal{G}_{i_{1},i_{2}}$ corresponding to two highlighted rows of $\mathcal{C}$.
Graphs obtained from circulant matrices are collections of cycles thus their chromatic number is at most $3$.\\}
\label{fig:cycle}
\vspace{-0.32in}
\end{figure}

\subsubsection{Toeplitz and Hankel matrices}
The associated $\mathcal{P}$-models are obtained in a similar way as for circulant matrices, in particular each column of each $\vv{P}_{i}$
is a binary vector. The corresponding coherence graphs have vertices of degrees at most $2$ and thus the chromatic number $\chi[\mathcal{P}]$
is at most $3$. As for the previous case, coherence $\mu[\mathcal{P}]$ is of the order $O(1)$ and $\tilde{\mu}[\mathcal{P}] = 0$.

\subsubsection{Fastfood matrices}
The Fastfood~\cite{LSS13} approach is a very special case of the $\mathcal{P}$-model.
Note that the core term in the Fastfood transform, Eqn.~\ref{eq:fastfood}, is the structured matrix $\textbf{H}\textbf{G}$, where $\textbf{H}=\{{h}_{i,j}\}$ is Hadamard and $\textbf{G}$ 
is a random diagonal gaussian matrix  (the rightmost terms $\vv{H}\vv{B}$ in Eqn.~\ref{eq:fastfood} implement data preprocessing 
to make all datapoints dense, and normalization is implemented by the leftmost scaling matrix $\vv{S}$). The matrix $\textbf{H}\textbf{G}$ can be constructed via the $\mathcal{P}$-model
with the fixed budget of randomness $\vv{g}=(g_{1},...,g_{n})$ and using the sequence of matrices 
$\mathcal{P}=(\textbf{P}_{1},...,\textbf{P}_{n})$, where each $\textbf{P}_{i}$ is a random diagonal matrix with
entries on the diagonal of the form: $h_{i,1},...,h_{i,n}$.
The quality of the FastFood approach can be now explained in the general $\mathcal{P}$-model method framework. One can easily see that the graphs  related to the model are empty (since $\vv{P}^{T}_{i,n_{1}} \vv{P}_{j,n_{2}} = 0$ for $n_{1} \neq n_{2}$).
The sublinearity of $\tilde{\mu}[\mathcal{P}]$ comes from the fact that  with high probability any two rows of $\vv{H}\vv{G}$ are close to be orthogonal.

\subsubsection{Toeplitz-like semi-Gaussian matrices}
\label{subsec:semi}

Consider Toeplitz-like matrices expressible by Eqn.~\ref{eq:circ_skewcirc} with displacement rank $r$. 
We will assume that $\vv{g}^{1},...,\vv{g}^{r} \in \mathbb{R}^{n}$ defining the Circulant-components in Eqn.~\ref{eq:circ_skewcirc} are independent Gaussian vectors. They will serve as a
``budget of randomness'' in the related $\mathcal{P}$-model that we are about to describe, with $r$ allowing a tunable tradeoff between structure and randomness. The vectors $\vv{h}^{1},...,\vv{h}^{r}$defining the skew-Circulant components in Eqn.~\ref{eq:circ_skewcirc} can be defined in different ways.  Below we present two general schemes:

\textbf{Random discretized vectors $\vv{h}^{i}$:} Each dimension of each $\vv{h}^{i}$ is chosen independently at random from the 
binary set $\{-\frac{1}{\sqrt{nr}},\frac{1}{\sqrt{nr}}\}$.

\textbf{Sparse setting:} Each $\vv{h}^{i}$ is sparse (but nonzero), i.e. has only few nonzero entries.
Furthermore, the sign of each $\vv{h}^{i}_{j}$ is chosen independently at random and
the following holds: $\|\vv{h}^{1}\|^{2}+...+\|\vv{h}^{r}\|^{2}=1$. 
This setting is characterized by a parameter $\kappa$ defining the 
size of the set of dimensions that are nonzero for at least one
$\vv{h}^{i}$.

 
We refer to such matrices as Toeplitz-like semi-Gaussian matrices. We now sketch how they can be obtained from the $\mathcal{P}$-model.
We take $t=nr$ and $\vv{g} = (g^{1}_{1},...,g^{1}_{n},...,g^{r}_{1},...,g^{r}_{n})^{T}$.
The matrix $\vv{P}_{1}$ is constructed by vertically stacking $r$ matrices $\vv{S}_{j}$ for $j=1,...,r$,
where each $\vv{S}_{j}$ is constructed as follows. The first column of $\vv{S}_{j}$ is $\vv{h}^{j}$ and the
subsequent columns are obtained from previous by skew-Circulant downward shifts.
Matrix $\vv{P}_{i}$ for $i>1$ is obtained from $\vv{P}_{i-1}$ by upward Circulant shifts, independently for each column at each block $\vv{S}_{j}$.

Matrices constructed according to this procedure satisfy conditions regarding certain structural parameters of the $\mathcal{P}$-model (see: Theorem \ref{ldrm-extra}).
In particular, in the sparse semi-Gaussian setting the corresponding coherence graphs have vertices of degrees bounded by a constant; thus,
by Lemma \ref{graph_lemma} the $\mathcal{P}$-models associated with them have low chromatic numbers.


\subsection{Construction of Random Feature Maps}
\label{sec:alg}

Given $S[\P]$, the $m\times n$ structured random matrix defined by a $\P$-model,  in lieu of using the $k \times n$
Gaussian random matrix $\vv{M}$ in Eqn.~\ref{eq:featuremap},  the feature map for a data vector $\vv{x}$ is constructed as follows.
\begin{compactitem}
\item Preprocessing phase: Compute $\vv{x}' = D_{1}HD_{0}\vv{x}$, where $H \in \mathbb{R}^{n \times n}$ is a $l_{2}$-normalized Hadamard matrix and  $D_{0},D_{1} \in \{-1,+1\}^{n \times n}$ are independent random diagonal matrices. Note that this transformation does not change the values of Gaussian or Arc-cosine kernels, since they are spherically-invariant. This preprocessing densifies the input data vector.
\item Compute $\vv{x}'' = S[\P] \vv{x} \in \reals^m$.
\item Compute $\vv{\bar{x}} \in \reals^k$ by concatenating random instantiations of the vector $\vv{x}''$ above obtained from $k/m$ independent constructions of $S[\P]$. 
\item Return $ \Psi(\vv{x}) = \frac{1}{\sqrt{k}} s(\vv{\bar{x}})$
\end{compactitem}

Note that the displacement rank $r$ for low displacement rank matrices and the  number of rows $m$ of a single structured block can be used to control the ``budget of randomness";  $m=1$ reduces to a completely unstructured matrix.

\section{Theoretical results}
\label{sec:theory}

In this section we provide concentration results regarding $\mathcal{P}$-model for Gaussian and arc-cosine kernels, 
showing in particular that the variance of the computed structured approximation of the kernel
is close to the unstructured one. We also present results targeting specifically low displacement rank structured matrices, and show how the displacement rank knob can be used to increase the budget of randomness and reduce the variance.

Let us denote by $\tilde{\K}_{\mathcal{P}}(\textbf{x},\textbf{z})$ the approximation of the kernel for two vectors
$\vv{x},\vv{z} \in \mathbb{R}^{n}$ if the $\mathcal{P}$-model is used. 
By $\tilde{\K}_{\vv{G}}(\textbf{x},\textbf{z})$ we denote the approximation of the kernel for two vectors
$\vv{x},\vv{z} \in \mathbb{R}^{n}$ if the fully unstructured setting with truly random Gaussian matrix $\vv{G}$ is applied.
All the proofs are in the Appendix. We start with the following result.
\begin{lemma}[Unbiasedness of the $\mathcal{P}$-model]
\label{unbiasedness_lemma}
Presented $\mathcal{P}$-model mechanism gives an unbiased estimation of the Gaussian and $b^{th}$-order arc-cosine kernels for $b \in \{0,1\}$
if for every $\vv{P}_{i}$ any two different columns $\vv{P}_{i,j}$,$\vv{P}_{i,k}$ of $\vv{P}_{i}$ satisfy $\vv{P}^{T}_{i,j}\vv{P}_{i,k}=0$. 
Thus,  $\mathbb{E}[\tilde{\K}_{\mathcal{P}}(\textbf{x},\textbf{z})] = \K(\vv{x},\vv{z}).$
\end{lemma}

The orthogonality condition $\vv{P}^{T}_{i,j}\vv{P}_{i,k}=0$ is trivially satisfied by Hankel, circulant or Toeplitz structured matrices produced by the $\mathcal{P}$-model
as well as Toeplitz-like semi-Gaussian matrices, where each $\vv{h}^{i}$ has one nonzero entry.
It is also satisfied in expectation (which in practice suffices) for all presented Toeplitz-like semi-Gaussian matrices.

For a $\mathcal{P}$-model, where matrices $\vv{P}_{i}$ were chosen randomly we denote as $\eta[\mathcal{P}]$ the maximum possible value that 
a random variable  $(\vv{P}_{i,n_{1}}^{T}\vv{P}_{j,n_{1}})^{2}$ can take for $1 \leq i < j \leq m, 1 \leq n_{1} \leq n$.
Without loss of generality we will assume that data vectors are drawn from the ball $\mathcal{B}(0,1)$ centered at $0$ of unit $l_{2}$ norm. Below we state results regarding $d^{th}$ moments of the obtained kernel's approximation via the $\mathcal{P}$-model
that lead to the concentration results.

\begin{theorem}
\label{main_theorem_1}
Let $\vv{x},\vv{z} \in \mathcal{B}(0,1)$ and let $d \in \mathbb{N}$.
Assume that each structured block of a matrix $\vv{A}$ (see: Section \ref{sec:alg}) produced according to the $\mathcal{P}$-model has $m$ rows
and $\tilde{\mu}[\mathcal{P}] = o(\frac{n}{\log^{2}(n)})$.
If matrices $\vv{P}_{i}$ of the $\mathcal{P}$-model are chosen randomly then assume furthermore that for 
any $1 \leq i < j \leq m$ and $1 \leq n_{1} < n_{2} \leq n$ the $n_{1}^{th}$ 
column of $\vv{P}_{i}$ is chosen independently from the $n_{2}^{th}$ column of $\vv{P}_{j}$.
If matrices $\vv{P}_{i}$ are chosen deterministically then for any $T, \epsilon > 0$ the following is true for $n$ large enough:
$$|\mathbb{E}[\tilde{\K}_{\mathcal{P}}^{d}(\textbf{x},\textbf{z})]-\mathbb{E}[\tilde{\K}_{\vv{G}}^{d}(\textbf{x},\textbf{z})]| \leq O(p_{gen}(T) + p_{struct}(T) + d \epsilon),$$
where:
\begin{equation}
p_{gen}(T) = \frac{4d}{\sqrt{2 \pi T}} e^{-\frac{T}{2}} + 4ne^{-\frac{\log^{2}(n)}{8}},
\end{equation}
\begin{align}
\begin{split}
p_{struct}(T) =  4\sum_{i=1}^{m}\chi(i,i)e^{-\frac{1}{8\mu^{2}[\mathcal{P}]\chi^{2}[\mathcal{P}]}\frac{n}{\log^{6}(n)}}\\
+2\sum_{1 \leq i \leq j \leq m}\chi(i,j)e^{-\frac{\epsilon^{2}\sqrt{n}}{8\mu^{2}[\mathcal{P}]\chi^{2}[\mathcal{P}]T\log^{4}(n)}}\\
\end{split}
\end{align}
and expectations are taken in respect to random choice for a Gaussian vector $\vv{g}$.
If $\vv{P}_{i}s$ are chosen from the probabilistic model then the above holds with probability at least 
$1-p_{wrong}$ in respect to random choices of $\vv{P}_{i}s$, where 
$$p_{wrong} = 2\sum_{i \leq i < j \leq m}e^{-\frac{n}{8\log^{6}(n)\eta[\mathcal{P}]}}.$$
\end{theorem}

Let us comment on the result above. The upper bound is built from two main components: $p_{gen}$ and $p_{struct}$. 
The first one depends on the general parameters of the setting: dimensionality of the data $n$ and order of the computed moment $d$.
The second one is crucial to understand how the structure of the matrix influences the quality of the model. 
We can immediately see that low chromatic numbers $\chi(i,j)$ (see: Section \ref{sec:p_model}) improve quality since they decrease computed upper bound. 
Furthermore, low values of the coherence $\mu[\mathcal{P}]$ and chromatic number $\chi[\mathcal{P}]$ also
lead to stronger concentration results. Both observations were noticed by us before, but now we see how they are implied by general theoretical results.
Finally, for all considered settings, where matrices $\vv{P}_{i}$ are constructed randomly 
parameter $\eta[\mathcal{P}]$ is of order $O(1)$ thus $p_{wrong}$ in negligibly small. 

In particular, if both the chromatic number $\chi[\mathcal{P}]$ and the coherence $\mu[\mathcal{P}]$ are of the order $O(poly(\log(n)))$ then $p_{struct}$
if inversely proportional to the superpolynomial function of $n$ thus is negligible in practice. 
That, as we will see soon, will be the case for proposed Toeplitz-like semi-Gaussian matrices with sparse vectors $h^{i}$. 

Let us also note that Theorem \ref{main_theorem_1} can be straightforwardly applied to the structured matrix from the Fastfood model since the condition 
regarding $\tilde{\mu}[\mathcal{P}]$ is satisfied and so is the independence condition.
Since all the chromatic numbers are equal to zero (because corresponding graphs are empty), $p_{struct} = 0$ and thus the theorem holds.

Theorem \ref{main_theorem_1} implies also that variances of the kernel approximation for the structured $\mathcal{P}$-model case and
unstructured setting are very similar (we borrow denotation from Theorem \ref{main_theorem_1}).

\begin{theorem}
\label{main_variance_theorem}
Consider the setting as in Theorem \ref{main_theorem_1}. If matrices $\vv{P}_{i}$ are chosen deterministically then for any $T, \epsilon > 0$ 
the following is true for $n$ large enough:
\begin{equation}
|Var(\tilde{\K}_{\mathcal{P}}(\textbf{x},\textbf{z})) -Var(\tilde{\K}_{\vv{G}}(\textbf{x},\textbf{z}))|= O(\frac{m-1}{2k}\Delta),
\end{equation}
where $Var$ stands for the variance and $\Delta = p_{gen}(T) + p_{struct} + \epsilon$.
If $\vv{P}_{i}s$ are chosen from the probabilistic model then the above holds with probability at least 
$1-p_{wrong}$, where $p_{wrong}$ is as in Theorem \ref{main_theorem_1}.
\end{theorem}

Note that in practice it means that the variance in the structured and unstructured setting is similar.
In particular, choosing $\epsilon=O(\frac{1}{m^{2}})$, $T > 7\log(m)$, one can deduce that the variance in the structured setting is of the order $O(\frac{1}{m})$ for $n$ large enough (the well known fact is that the unstructured variance is of the order $O(\frac{1}{m})$).
Note also that as expected, for $m=1$ the structured setting becomes an unstructured one, since each structured block consists of just one row and different blocks are constructed independently.

{\bf Toeplitz-like semi-Gaussian Low-displacement rank matrices}: Note that the structure of a matrix affects only the $p_{struct}$ factor in the  statements above.  Thus, we will focus on the structured parameters of the $\mathcal{P}$-model.
We will show that Toeplitz-like semi-Gaussian matrices can be set up so that the above parameters are of required order. 

\begin{theorem}
\label{ldrm-intro}
Consider Toeplitz-like semi-Gaussian matrices with sparse skew-Circulant factors (as in Subsection \ref{subsec:semi}). Let $\kappa$ 
denote the number of dimensions that are nonzero for at least one $\vv{h}^{i}$. Then for $1 \leq i \leq j \leq m$ we have: $\chi(i,j) \leq \kappa^{2} + 1$.
Furthermore, $\mu[\mathcal{P}] \leq \kappa$ and the bound on 
$|\mathbb{E}[\tilde{\K}_{\mathcal{P}}^{d}(\textbf{x},\textbf{z})]-\mathbb{E}[\tilde{\K}_{\vv{G}}^{d}(\textbf{x},\textbf{z})]|$ derived in Theorem \ref{main_theorem_1}
is valid also here if $r \geq 3log^{5}(n)$ and for $p_{wrong}$ of the order $o(\frac{1}{n})$.
\end{theorem}

The richness of the low displacement rank mechanism comes from the fact that the budget of randomness can be controlled by the rank parameter $r$ and  increasing $r$ leads to better quality approximations. In particular, we have:

\begin{theorem}
\label{ldrm-extra}
Consider Toeplitz-like semi-Gaussian matrices with sparse skew-Circulant factors and parameter $\kappa$. Assume that each $\vv{h}^{i}$ has exactly $\alpha$ nonzero dimensions, 
each nonzero dimensions taken independently at random from $\{-\frac{1}{\alpha r}, \frac{1}{\alpha r}\}$. 
Then, 
$\mathbb{P}[|\mu[\mathcal{P}]| > \tau] \leq 4n^{2}e^{-\frac{\tau^{2}\alpha r}{O(\kappa^{2})}}.$
\end{theorem}

Note that increasing rank $r$ leads to sharper upper bounds on the coherence $\mu[\mathcal{P}]$ (in practice $r$ polynomial in $log(n)$ suffices) and thus, 
from what we have said so far, to better concentration results for the entire structured scheme. 
Analogous variance bounds can also be derived for Toeplitz-like semi-Gaussian matrices where the $\vv{h}^i$ vectors are chosen to be dense. But due to lack of space, these results are included in our supplementary material.

%% file: experiments.tex
\section{Empirical Support}
In this section, we compare feature maps obtained with fully Gaussian,  
Fastfood, Circulant, and Toeplitz-like matrices with increasing displacement rank. 
Our goal is to lend support to the theoretical contributions of this paper by showing that high-quality 
feature maps can be constructed from a broad class of structured matrices as instantiations of the proposed $\P$-model.

{\bf Kernel Approximation Quality}: In Figure~\ref{fig:g50c}, we report  relative Frobenius error in reconstructing the Gram matrix, i.e. $\frac{\|\vv{K} - \vv{\tilde{K}}\|_{fro}}{\|\vv{K}\|_{fro}}$ where $\vv{K}, \vv{\tilde{K}}$ denote the exact and approximate Gram matrices, as a function of the number of random features.   We use the g50c dataset which comprises of $550$ examples drawn from multivariate Gaussians in $50$-dimensional space with means separated such that the Bayes error is $5\%$.  We see that Circulant matrices and Toeplitz-like matrices with very low displacement rank (1 or 2) perform as well as  Fastfood feature maps. In all experiments, for Toeplitz-like matrices, we used skew-Circulant parameters (the $\vv{h}$ vectors in Eqn.~\ref{eq:circ_skewcirc}) with average sparsity of $5$. As the displacement rank is increased, the budget of randomness increases and the reconstruction error approaches that of Gaussian Random features, as expected based on our theoretical results.
\begin{figure}[h]
\begin{center}
\includegraphics[height=4cm,width=0.8\linewidth]{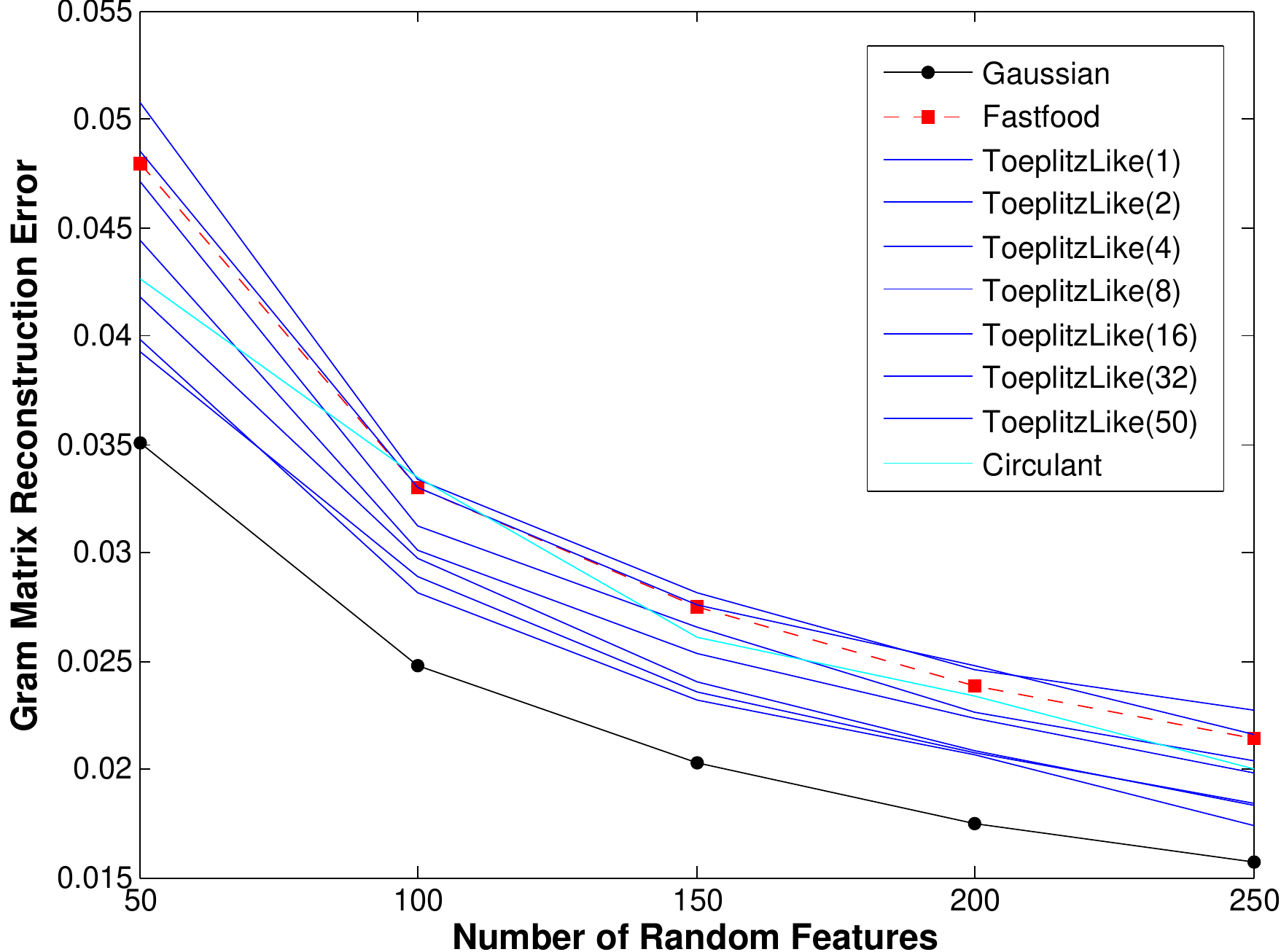}\\\label{fig:g50c}
\end{center}
\caption{Lower blue curves (better reconstruction) correspond to Toeplitz-like matrices with increasing displacement rank.}
\end{figure}
\begin{table*}[t]
\begin{center}
\caption{Kernel approximation (first row) and classification error (second row) in percentage for Complex Exponential (Gaussian Kernel).}\label{tab:complexexp}
{\scriptsize
\begin{tabular}{cccc|cccccc}
 & & Gaussian & QMC (Halton) & Fastfood & Circulant & ToeplitzLike(1) & ToeplitzLike(5) & ToeplitzLike(10) & ToeplitzLike(20)\\
\hline
 & \multirow{2}{*}{USPS (k=256)}  & 5.06 & 5.05 & 6.76 & 7.61 & 9.66 & 7.55  & 6.86 & {\bf 6.68}\\
 & & 7.12 & 6.90 &  7.37 &  7.54 & 7.72 & 7.44  & 7.46  & {\bf 7.29}\\
& \multirow{2}{*}{USPS (k=1280)} & 2.32&    2.15 &       3.06 &      3.32 &       4.41&           3.35&     3.16&     {\bf 3.00}\\
	 & 	&			4.52&      4.73&       4.62&      4.53&       4.62&           4.58 &           {\bf 4.53}            & 4.65\\
\hline 

&\multirow{2}{*}{DNA (k=80)} & 3.6 & 3.51 & 5.01 & 4.62 & 6.26 & 4.65 & 4.40 & {\bf 4.10} \\
  &                             &31.04  & 30.94 & 31.04 & 30.94 &  31.35 &  30.82  &{\bf 30.29}   & 30.70 \\
&\multirow{2}{*}{DNA ($k=900$)} & 1.61 & 1.59 &       2.23 &      2.06 &       2.88&           2.09 &           1.93 &            {\bf 1.83} \\
   &                             &16.5  & 15.01 &        16.94 &      16.63 &        16.82     &      {\bf 16.34} &           16.57 &            16.57\\
   \hline
&\multirow{2}{*}{COIL ($k=1024$)} & 2.74 & 2.41 & {\bf 3.67} & 4.45 & 5.60 & 4.47  & 4.09 & 3.79\\
&							  & 0.52 & 1.11 &  {\bf 0.49} &  0.62 & 0.62 & 0.48  & 0.57  & 0.52\\
	& \multirow{2}{*}{COIL ($k=2048$)}  & 1.92 &     1.87  &     {\bf 2.64} &      3.14 &      4.18 &          3.04 &           2.87 &             2.76\\  
& & 0.17  &   0.28 &      {\bf 0.15} &     0.19 &      0.19 &          0.20    &      0.19 &           0.19\\
								  \hline		  
\end{tabular}
}\end{center}
\end{table*}
Results on publicly available real-world classification datasets, averaged over $100$ runs, are reported in Table~\ref{tab:complexexp} for complex exponential nonlinearity (Gaussian kernel). Results with ReLU (arc-cosine) are similar but not shown for lack of space.  As observed in previous papers, better Gram matrix approximation is not often correlated with higher classification accuracy. Nonetheless, it is clear that the design of space of valid feature map constructions based on structured matrices is much larger than what has so far been explored in the literature:  Circulant and Toeplitz-like matrices are very competitive with Fastfood, and sometimes give better results particularly with increasing displacement rank. The effectiveness of such feature maps for nonlinearities other than the complex exponential also validates our theoretical contributions. Among the unstructured baselines, we also include Quasi-Monte Carlo (QMC) feature maps of~\cite{qmc} using Halton low-discrepancy sequences. The use of structured matrices to accelerate QMC techniques building on~\cite{DickFastQMC} is of interest for future work.

\begin{figure}[h]
\begin{center}
\caption{Lower blue curves (smaller speedup) correspond to Toeplitz-like matrices with increasing displacement rank.}
\includegraphics[height=4.3cm, width=0.8\linewidth]{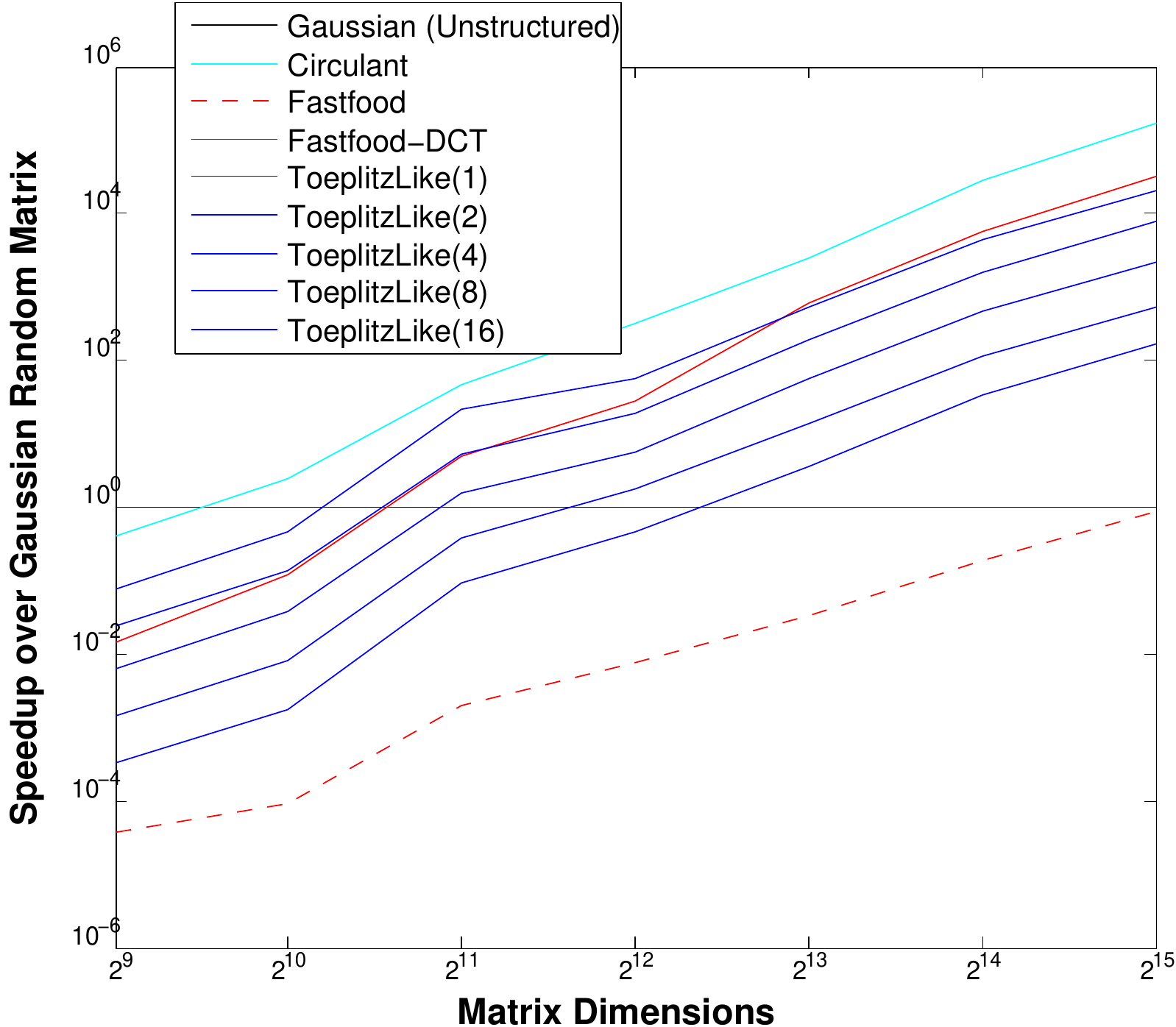}\label{fig:speedup}
\end{center}
\end{figure}

{\bf Speedups}: Figure~\ref{fig:speedup} shows the speedup obtained in featuremap construction time using structured matrices relative to 
using unstructured  Gaussian random matrices (on a 6-core 32-GB Intel(R) Xeon(R) machine running Matlab R2014a). The benefits of 
sub-quadratic matrix-vector multiplication with FFT-variations tend to show up beyond $1024$ dimensions. Circulant-based feature 
maps are the fastest to compute. Fastfood (with DCT instead of Hadamard matrices) is about as fast as Toeplitz-like matrices with 
displacement rank 1 or 2. Higher displacement rank matrices show speedups at higher dimensions as expected. Fastfood with inbuilt 
\texttt{fwht} routine in Matlab performed poorly in our experiments.

%% file: appendix.tex
\section{Appendix}

We now prove all theoretical results of the paper.
We need to introduce some technical denotation.

From now on $f$ denotes one from the following functions: $\sin$, $\cos$, $sign$ or a linear rectifier. We call the set of these functions $\mathcal{F}$.
For two vectors $v,w$ we denote by $v \cdot w$ their dot product.
We denote by $G_{struct}^{i}$ for $i=1,...,\frac{k}{m}$ the building blocks of the structured matrix constructed according to the $\mathcal{P}$-model that are vertically stacked to 
produce the final structured matrix.
Let $v^{1},v^{2} \in \mathbb{R}^{n}$ be two datapoints from the preprocessed input-dataset $D_{1}HD_{0}\mathcal{X}$. 
Let $d$ be a fixed integer constant.
Let $R=\{i_{1},...,i_{r}\}$ be some $r$-element subset of the set $\{1,...,m\}$, where $m$ stands for the number of rows used in the construction of matrices $G_{struct}^{i}$ (key building blocks of our structured mechanism). Finally, let $\alpha_{1},...,\alpha_{r}$ be positive integers such that
$\alpha_{1}+...+\alpha_{r} = d$. 
\begin{definition}
For three vectors: $v,w,z \in \mathbb{R}^{n}$ and a given nonlinear function $f \in \mathcal{F}$ we denote:
$$\phi(v,w,z) = f(z \cdot v) f(z \cdot w).$$
\end{definition}
We will show that for a variety of functions $\Psi : \mathbb{R}^{r} \rightarrow \mathbb{R}$  the expected value of the expression $T^{G,d}_{v^{1},v^{2}}(\mathcal{R}, \alpha_{1},...,\alpha_{r})$ given by the formula:
\begin{equation}
\label{eq-unstructured}
\Psi(\phi_{1}(v^{1},v^{2},g^{i_{1}})^{\alpha_{1}},...,\phi_{r}(v^{1},v^{2},g^{i_{r}})^{\alpha_{r}}),
\end{equation}
where $g^{1},...,g^{m}$ is the set of $m$ gaussian vectors forming gaussian matrix $G$, each obtained by sampling independently $n$ values from the distribution $\mathcal{N}(0,1)$ and $\phi_{i}$s differ by the choice of nonlinear mapping $f_{i} \in \mathcal{F}$, can be accurately approximated by its structured version $T^{A,d}_{v^{1},v^{2}}((\mathcal{R}, \alpha_{1},...,\alpha_{r})$
which is of the form:
\begin{equation}
\label{eq-structured}
\Psi(\phi_{1}(v^{1},v^{2},a^{i_{1}})^{\alpha_{1}},...,\phi_{r}(v^{1},v^{2},a^{i_{r}})^{\alpha_{r}}),
\end{equation}
where $a^{1},...,a^{m}$ are rows of the structured matrix $A = G_{struct}^{i}$.
The importance of $T^{G,d}_{v^{1},v^{2}}(\mathcal{R}, \alpha_{1},...,\alpha_{r})$ and $T^{A,d}_{v^{1},v^{2}}(\mathcal{R}, \alpha_{1},...,\alpha_{r})$  lies in the fact that $d^{th}$ moments of the random variables approximating considered kernels in the unstructured and structured mechanism can be expressed as weighted sums of the expressions of the form $T^{G,d}_{v^{1},v^{2}}(\alpha_{1},...,\alpha_{r})$
and $T^{A,d}_{v^{1},v^{2}}(\alpha_{1},...,\alpha_{r})$ respectively if $\Psi(x_{1},...,x_{r}) = x_{1} \cdot ... \cdot x_{r}$.
Thus if $T^{A,d}_{v^{1},v^{2}}(\alpha_{1},...,\alpha_{r})$ closely approximates $T^{G,d}_{v^{1},v^{2}}(\alpha_{1},...,\alpha_{r})$ then the corresponding moments are similar. That, as we will see soon, implies several theoretical guarantees for the structured method. In particular, this means that the variances are similar. Since in the unstructured setting the variance is of the order $O(\frac{1}{m})$, that will be also the case for the structured setting. This in turn will imply concentration results providing theoretical explanation for the observations from the experimental section that show the quality of the proposed structured setting.

We need to introduce a few definitions.

\begin{definition}
We denote by $\Delta^{\xi}_{s}$ the supremum of the expression $\|\xi(y_{1},...,y_{m}) - \xi(y^{\prime}_{1},...,y^{\prime}_{m})\|$ over all pairs of vectors $(y_{1},...,y_{m}), (y^{\prime}_{1},...,y^{\prime}_{m})$ from the domain $\mathcal{D}$ that differ on at most one dimension and by at most $s$.
We say that a function $\xi : \mathbb{R}^{m} \rightarrow \mathbb{R}$ is $M$-bounded in the domain $\mathcal{D}$ if $\Delta^{\xi}_{\infty} = M$.
\end{definition}

Note that the value of the function $\phi_{i}(v^{1},v^{2},g^{i})^{\alpha_{i}}$
depends only on the projection $g^{i}_{proj}$ of $g^{i}$ on the $2$-dimensional space spanned by $v^{1}$ and $v^{2}$. Thus for a given pair $v^{1},v^{2}$ function $\phi$ is in fact a function $B^{v^{1},v^{2}}_{i}$ of this projection.

\begin{definition}
Define:
\begin{align}
\begin{split}
p_{\lambda,\epsilon} = \sup_{i, v^{1},v^{2}, \|\zeta|_{\infty} \leq \epsilon}
\mathbb{P}[|B^{v^{1},v^{2}}_{i}(g^{i}_{proj}+\zeta)-\\
B^{v^{1},v^{2}}_{i}(g^{i}_{proj})| > \lambda],
\end{split}
\end{align}
where the supremum is taken over all indices $i=1,...,m$, all pairs of linearly independent vectors from the domain, all coordinate systems in $span(v^{1},v^{2})$ and vectors $\zeta$ of $L_{1}$-norm at most $\epsilon$ in some of these coordinate systems.
\end{definition}

We will use the following notation: $\sigma_{i,j}(n_{1},n_{2}) = \vv{P}_{i,n_{1}}^{T}\vv{P}_{j,n_{2}}$.
To compress the statements of our theoretical results, we will use also the following notation:
$$\xi(i_{i},i_{2}) = 2\chi(i_{1},i_{2}) \sqrt{\sum_{1 \leq n_{1} < n_{2} \leq n} (\sigma_{i_{1},i_{2}}(n_{1},n_{2}))^{2}},$$ 
We will also denote: 
$\lambda(i_{1},i_{2}) = \sum_{j=1}^{n}|\sigma_{i_{1},i_{2}}(j,j)|$ and
$\tilde{\lambda}(i_{1},i_{2}) = |\sum_{j=1}^{n}\sigma_{i_{1},i_{2}}(j,j)|$
for $1 \leq i_{1} \leq i_{2} \leq m$ (see: \ref{sec:p_model}).

Note first that the preprocessing step preserves kernels' values since transformation $HD_{0}$
is an isometry and considered kernels are spherically-invariant.
We start with Lemma \ref{unbiasedness_lemma}.

\begin{proof}
Note that it suffices to show that for any two given vectors $x,y \in \mathbb{R}^{n}$ the following holds: 
\begin{equation}
\label{mean_eq_1}
\mathbb{E}[f(G^{i}_{struct}x) \cdot f(G^{i}_{struct}y)]=\mathbb{E}[f(Gx) \cdot f(Gy)],
\end{equation}
where $G$ is the unstructured gaussian matrix. Let $g^{i,j}_{struct}$ be the $j^{th}$ row of $G^{i}_{struct}$
and let $g^{j}$ be the $j^{th}$ row of $G$. Note that we have:
\begin{equation}
\label{mean_eq_2}
\mathbb{E}[f(g^{i,j}_{struct} \cdot x) f(g^{i,j}_{struct} \cdot y)]=\mathbb{E}[f(g^{j} \cdot x)f(g^{j} \cdot y)].
\end{equation}
The latter follows from the fact that $g^{i,j}_{struct}$ has the same distribution as $g$. To see this note that
$g^{i,j}_{struct} = g \cdot P_{i}$. Thus dimensions of $g^{i,j}_{struct}$ are projections of $g$ onto columns of $P_{i}$.
Each projection is trivially gaussian from $\mathcal{N}(0,1)$ (that is implied by the fact that each column is normalized).
The independence of different dimensions of $g^{i,j}_{struct}$ comes from the observation that different columns are orthogonal.
Thus we can use a simple property of gaussian vectors stating that the projections of a gaussian vector on mutually orthogonal directions are
independent. The equation \ref{mean_eq_1} implies equation \ref{mean_eq_2} by the linearity of expectation and that completes the proof.
\end{proof}

Now we prove Theorem \ref{main_theorem_1}.
This one is easily implied by a more general result that we state below.
We will assume that function $\Psi$ from equations: \ref{eq-unstructured}, \ref{eq-structured} is $M$-bounded for some given $M>0$.
We will assume that expected values defining $T^{A,d}$ are not with respect to the random choices determining $P_{i}s$.

\begin{theorem}
\label{main_theorem}
Let $v^{1},v^{2} \in \mathbb{R}^{n}$ be two vectors from a dataset $\mathcal{X}$.
Let $\mathcal{R}=\{i_{1},...,i_{r}\} \in \{1,...,m\}$ and let $\alpha_{1},...,\alpha_{r}$ be the set
of positive integers such that $\alpha_{1} + ... + \alpha_{r} = d$.
Assume that each structured matrix $G_{struct}^{i}$ consists of $m$ rows
and either $\sup_{1 \leq i_{1} < i_{2} \leq m} \lambda(i_{1},i_{2}) = o(\frac{n}{\log^{2}(n)})$
if $P_{i}s$ were constructed deterministically or $sup_{1 \leq i_{1} < i_{2} \leq m} \mathbb{E}[\tilde{\lambda}(i_{1},i_{2})] = 
o(\frac{n}{\log^{2}(n)})$ if $P_{i}s$ were constructed randomly.
In the latter case assume also that for any $1 \leq i_{1} < i_{2} \leq m$ and $1 \leq n_{1} < n_{2} \leq n$ the $n_{1}^{th}$ 
column of $P_{i_{1}}$ is chosen independently from the $n_{2}^{th}$ column of $P_{i_{2}}$.
Denote by $\Psi_{max}$ the maximal value of the function $\Psi$ for the datapoints from $\mathcal{X}$.
Let $q_{v^{1},v^{2}}^{d} = |T^{A,d}_{v^{1},v^{2}}(\mathcal{R},\alpha_{1},...,\alpha_{r})-T^{G,d}_{v^{1},v^{2}}(\mathcal{R},\alpha_{1},...,\alpha_{r})|$ denote the absolute value of the difference of the two fixed terms on the weighted sum for the $d$-moments of the kernel's approximation in the structured $\mathcal{P}$-model setting and the fully unstructured setting. 
Then for any $\lambda, \epsilon > 0$, $T>0$, $n$ large enough and $P_{i}s$ chosen deterministically we have:
$$q_{v^{1},v^{2}}^{d} \leq 
(p_{gen} + p_{struct})\Psi_{max} + \sum_{i=0}^{d} p_{f}^{i}(iM + (d-i)\Delta^{\Psi}_{\lambda}),$$ where:
\begin{equation}
p_{gen} = \frac{4r}{\sqrt{2 \pi T}} e^{-\frac{T}{2}} + 4ne^{-\frac{\log^{2}(n)}{8}},
\end{equation}
\begin{equation}
p_{f}^{i} = {d \choose i} (p_{\lambda, \epsilon})^{i}
\end{equation}
and 
\begin{align}
\begin{split}
p_{struct} =  4\sum_{i=1}^{m}\chi(i,i)e^{-\frac{1}{2\xi^{2}(i,i)}\frac{n^{2}}{\log^{6}(n)}}\\
+2\sum_{1 \leq i_{1} \leq i_{2} \leq m}\chi(i_{1},i_{2})e^{-\frac{\epsilon^{2}n^{\frac{3}{2}}}{2\xi^{2}(i_{1},i_{2})T\log^{4}(n)}}\\
\end{split}
\end{align}
If $P_{i}s$ are chosen from the probabilistic model then the above holds with probability at least 
$1-p_{wrong}$, where $p_{wrong} = 2\sum_{i \leq i_{1} < i_{2} \leq m}e^{-\frac{n^{2}}{8\log^{6}(n)\sum_{j=1}^{n}(\sigma_{i_{1},i_{2}}(j,j))^{2}}}$.
\end{theorem}

\begin{proof}
Consider the expression $$q^{d}_{v_{1},v_{2}} = |T^{A,d}_{v^{1},v^{2}}(\mathcal{R}, \alpha_{1},...,\alpha_{r}) - 
T^{G,d}_{v^{1},v^{2}}(\mathcal{R}, \alpha_{1},...,\alpha_{r})|.$$ 
We will use formulas for $T^{G,d}$ and $T^{A,d}$ given by equations: \ref{eq-unstructured} and \ref{eq-structured}.
Without loss of generality we will assume that $A=G^{i}_{struct}D_{1}$ i.e. in our theoretical analysis we will make $D_{1}$
a part of the structured mechanism and move it away from the preprocessing phase (obviously both ways are equivalent because of the 
associative property of matrix mutliplication).
We have already noted that each argument of the function $\Psi$ from equations: \ref{eq-unstructured} and \ref{eq-structured}
depends only on the projections of $a^{i_{1}},...,a^{i_{r}}$ on the $2$-dimensional space spanned by $v^{1}$ and $v^{2}$.
Denote these projections as: $a^{i_{1}}_{proj}$,...,$a^{i_{r}}_{proj}$ respectively and fix some orthonormal basis $\mathcal{B}$ of this $2$-dimensional space. 
As we will see soon, in the $\mathcal{P}$-model setting the coordinates of $a^{i}_{proj}s$ in $\mathcal{B}$ can be expressed as $g \cdot s^{i,j}$ for $j=1,2$, 
where $g$ is a vector representing a budget of randomness of the corresponding $\mathcal{P}$-model and $s^{i,j}$s are some vectors from $\mathbb{R}^{t}$
(parameter $t$ stands for the length of $g$).

We will show that $s^{i,j}$s, even though not necessarily pairwise orthogonal, are close to be pairwise orthogonal with high probability.
Let us assume now that vectors $s^{i,j}$ can be chosen in such a way that each $s^{i,j}$ satisfies: $s^{i,j} = w^{i,j} + \rho(i,j)$,
where vectors $w^{i,j}$ are mutually orthogonal, we have $\|s^{i,j}\|_{2} = \|w^{i,j}\|_{2}$ and furthermore  
$\|\rho(i,j)\|_{2} \leq \rho$ for some given $\rho > 0$. We call this property the $\rho$-orthogonality property.
We will later show that the $\rho$-orthogonality property depends on the random diagonal matrix $D_{1}$.

Assume now that the $\rho$-orthogonality property is satisfied. Denote by $g^{\mathcal{H}}$ the projection of the ``budget-of-randomness'' 
vector $g$ onto $2r$-dimensional linear space $\mathcal{H}$ spanned by vectors from $\{s^{i,j}\}$.
Note that then the coordinates of $a^{i}_{proj}s$ in $\mathcal{B}$ can be rewritten as $g \cdot w^{i,j} + \epsilon(i,j)$,
where $|\epsilon(i,j)| \leq \epsilon$ and $\epsilon = \|g^{\mathcal{H}}\|_{2} \rho$.
Thus each $\psi_{i}$ in the formula from equation \ref{eq-structured} can be then
expressed as $B_{i}^{v^{1},v^{2}}(g^{i}_{proj} + \epsilon(i))$, where $g^{i}_{proj}s$ stand for the
projections onto $2$-dimensional linear space spanned by $v^{1}$ and $v^{2}$ of independent copies of gaussian vectors $g^{i}$. Each $g^{i}$ is of the same distribution as the corresponding
structured vector $a^{i}$ and $\epsilon(i)s$ are vectors with the $L_{1}$-norm satisfying
$\|\epsilon(i)\| \leq \epsilon$. The independence comes from the fact that variables of the form $g \cdot w^{i,j}$ are independent. That, as in the 
proof of Lemma \ref{unbiasedness_lemma} is implied by the well known fact that dot products of a given gaussian vector with orthogonal vectors
are independent. 
Note that if not the term $\epsilon(i)$ then the formula for $T^{A,d}$ would collapse to its unstructured counterpart $T^{G,d}$.
We will argue that both expressions are still close to each other if $\epsilon(i)$ have small $L_{1}$-norm.

Let us fix $\lambda > 0$. Our goal is to count these indices $i$ that satisfy the following: 
$|\psi_{i}(v^{1},v^{2},g^{i})^{\alpha^{i}} - \psi_{i}(v^{1},v^{2},g^{i})^{\alpha^{i}}| > \lambda$,
where $g^{i}s$ corresponds to the aforementioned independent counterparts of $a^{i}s$. We call them 
\textit{bad indices}.
Based on what we have said so far, we can conclude that the latter inequality can be expressed as 
$|B_{i}^{v^{1},v^{2}}(g^{i}_{proj} + \epsilon(i)) - B_{i}^{v^{1},v^{2}}(g^{i}_{proj})| > \lambda$.
Let us first find the upper bound on the probability of the event that the number of bad indices is
$j$ for some fixed $1 \leq j \leq d$. Note that since $g^{i}s$ are independent, we can use Bernoulli scheme
to find that upped bound. Using the definition of $p_{\lambda, \epsilon}$ we obtain an upper bound of the form
$p_{upper} \leq {d \choose j} (p_{\lambda, \epsilon})^{j}$. If the number of bad indices is $j$ then 
by the definition of $M$ and $\Delta^{\Psi}_{\lambda}$ we see that $T^{A,d}$ differs from $T^{G,d}$ by at most
$iM + (d-i)\Delta^{\Psi}_{\lambda}$. Summing up over all indices $j$ we get the second term of the upper bound on
$q^{d}_{v^{1},v^{2}}$ from the statement of the theorem.

However the $\rho$-orthogonality does not have to hold. Note that (by the definition of $\Psi_{max}$) to finish the proof 
of the theorem it suffices to show that the probability of $\rho$-orthogonality not to hold is at most $p_{gen} + p_{struct}$.

\begin{lemma}
\label{technical_lemma}
The $\rho$-orthogonality property holds with probability at least $1 - (p_{gen} + p_{struct})$.
\end{lemma}

\begin{proof}
We need the following definition.
\begin{definition}
Let $x=(x_{1},...,x_{n})$ be a vector with $\|x\|_{2}=1$. We say that $x$ is $\theta$-balanced if $|x_{i}| \leq \frac{\theta}{\sqrt{n}}$
for $i=1,...,n$.
\end{definition}
For a fixed pair of vectors $v^{1},v^{2} \in \mathcal{X}$ choose some orthonormal basis $\mathcal{B} = \{x^{1},x^{2}\}$
of the $2$-dimensional space spanned by $v^{1}$ and $v^{2}$.
Let $\tilde{x}^{1}$ and $\tilde{x}^{2}$ be the images of $x^{1}$ and $x^{2}$ under transformation $HD_{0}$, where $H$
is a Hadamard matrix and $D_{0}$ is a random diagonal matrix.
We will show now that with high probability $\tilde{x}^{1}$ and $\tilde{x}^{2}$ are $\log(n)$-balanced.
Indeed, the $i^{th}$ dimension of $\tilde{x}^{1}$ is of the form:
$\tilde{x}^{1}_{i} = h_{i,1}x^{1}_{1} + ... + h_{i,n}x^{1}_{n}$, where
$h_{i,j}$ stands for the entry in the $i^{th}$ row and $j^{th}$ column of a matrix $HD_{0}$.
We need to find a sharp upper bound on $\mathbb{P}[|h_{i,1}x^{1}_{1} + ... + h_{i,n}x^{1}_{n}| \geq a]$ for $a = \frac{\log(n)}{\sqrt{n}}$.

We will use the following concentration inequality, calles \textit{Azuma's inequality}
\begin{lemma}
Let $X_{1},...,X_{n}$ be a martingale and assume that $-\alpha_{i} \leq X_{i} \leq \beta_{i}$ for some positive constants $\alpha_{1},...,\alpha_{n}, \beta_{1},...,\beta_{n}$. 
Denote $X = \sum_{i=1}^{n} X_{i}$.
Then the following is true:
$$\mathbb{P}[|X - \mathbb{E}[X]| > a] \leq 2e^{-\frac{a^{2}}{2\sum_{i=1}^{n}(\alpha_{i} + \beta_{i})^{2}}}$$
\end{lemma}
In our case $X_{j} = h_{i,j}x^{1}_{j}$ and $\alpha_{i} = \beta_{i} = \frac{1}{\sqrt{n}}$.
Applying Azuma's inequality, we obtain the following bound:
$\mathbb{P}[|h_{i,1}x^{1}_{1} + ... + h_{i,n}x^{1}_{n}| \geq \frac{\log(n)}{\sqrt{n}}] \leq 2e^{-\frac{\log^{2}(n)}{8}}$.
The probability that all $n$ dimensions of $\tilde{x}^{1}$ and $\tilde{x}^{2}$ have absolute value at most $\frac{\log(n)}{\sqrt{n}}$ is, by the union bound, at least $p_{balanced} = 1 - 2n \cdot 2e^{-\frac{\log^{2}(n)}{8}} = 1-4n e^{-\frac{\log^{2}(n)}{8}}$. 
Thus this a lower bound on the probability that $\tilde{x}^{1}$ and $\tilde{x}^{2}$ are $\log(n)$-balanced.
We will use this lower bound later. Now note that it does not depend on the particular form of the structured matrix since it is only related to the preprocessing phase, where linear mappings $D_{0}$ and $H$ are applied.

For simplicity we will now denote $\hat{x}^{1}$ and $\hat{x}^{2}$ simply as $x^{1}$ and $x^{2}$, knowing these are the original vectors after applying linear transformation $HD_{0}$.
Let us get back to the projections of $a^{i}s$ onto $2$-dimensional linear space spanned by $v^{1}$ and $v^{2}$. 
Note that we have already noticed that $a^{i} \cdot x^{j}$ ($j=1,2$) is of the form $g \cdot s^{i,j}$ for some vector $s^{i,j} \in \mathbb{R}^{t}$, where $t$ is the size of the ``budget of randomness'' used in the given $\mathcal{P}$-model.
From the definition of the $\mathcal{P}$-model we obtain:
\begin{equation}
s^{i,j}_{l} = d_{1}p^{i}_{l,1}x^{j}_{1} + ... + d_{n}p^{i}_{l,n}x^{j}_{n}
\end{equation}
for $l=1,...,t$, where $s^{i,j}_{l}$ stands for the $l^{th}$ dimension of $s^{i,j}$,
$p^{i}_{l,k}$ is the entry in the $l^{th}$ row and $k^{th}$ column of $P_{i}$
and $d_{r}s$ are the values on the diagonal of the matrix $D_{0}$.
As we noted earlier, we want to show that $s^{i,j}s$ are close to be mutually orthogonal.
To do it, we will compute dot products $s^{i_{1},j_{1}} \cdot s^{i_{2},j_{2}}$.
We will first do it for $i_{1}=i_{2}$.
We have:
\begin{align}
\begin{split}
s^{i_{1},j_{1}} \cdot s^{i_{1},j_{2}} = x^{j_{1}}_{1}x^{j_{2}}_{1}\sum_{l=1}^{t}(p^{i_{1}}_{l,1})^{2}+...+x^{j_{1}}_{n}x^{j_{2}}_{n}\sum_{l=1}^{t}(p^{i_{1}}_{l,n})^{2}\\
+2 \sum_{1 \leq n_{1} < n_{2} \leq n} d_{n_{1}}d_{n_{2}}x^{j_{1}}_{n_{1}}x^{j_{2}}_{n_{2}}(\sum_{i=1}^{t} p^{i_{1}}_{l,n_{1}}p^{i_{2}}_{l,n_{2}})
\end{split}
\end{align}
Now we take advantage of the normalization property of the matrices $P_{i}$ and the fact that $x^{1}$ is orthogonal to $x^{2}$ and conclude that the first term on the RHS of the equation above is equal to $0$.
Thus we have:
\begin{equation}
s^{i_{1},j_{1}} \cdot s^{i_{1},j_{2}} = 2 \sum_{1 \leq n_{1} < n_{2} \leq n} d_{n_{1}}d_{n_{2}}x^{j_{1}}_{n_{1}}x^{j_{2}}_{n_{2}}\sigma_{i_{1},i_{1}}(n_{1},n_{2}).
\end{equation}

Note that if for any fixed $P_{i}$ any two different columns of $P_{i}$ are orthogonal then
$\sigma_{i_{1},i_{1}}(n_{1},n_{2}) = 0$ and thus $s^{i_{1},j_{1}} \cdot s^{i_{1},j_{2}} = 0$.
This is the case for many structured matrices constructed according to the $\mathcal{P}$-model, for instance circulant, Toeplitz or Hankel matrices.

Let us consider now $s^{i_{1},j_{1}} \cdot s^{i_{2},j_{2}}$ for $i_{1} \neq i_{2}$.
By the previous analysis, we obtain:
\begin{align}
\begin{split}
s^{i_{1},j_{1}} \cdot s^{i_{2},j_{2}} = \sigma_{i_{1},i_{2}}(1,1)x^{j_{1}}_{1}x^{j_{2}}_{1}+...+\sigma_{i_{1},i_{2}}(n,n)x^{j_{1}}_{n}x^{j_{2}}_{n}\\
+ 2 \sum_{1 \leq n_{1} < n_{2} \leq n} d_{n_{1}}d_{n_{2}}x^{j_{1}}_{n_{1}}x^{j_{2}}_{n_{2}}\sigma_{i_{1},i_{2}}(n_{1},n_{2}).
\end{split}
\end{align} 
This time in general we cannot get rid of the first term in the RHS expression. This can be done if columns of the same indices in different $P_{i}s$ are orthogonal. This is in fact again the case for circulant, Toeplitz or Hankel matrices.

Let us now fix some $1 \leq i_{1} \leq m$ and $\kappa > 0$. 
Our goal is to find an upper bound on the following probability: $\mathbb{P}[|s^{i_{1},j_{1}} \cdot s^{i_{2},j_{2}}| > \kappa]$.

We have:
\begin{align}
\begin{split}
\mathbb{P}[|s^{i_{1},j_{1}} \cdot s^{i_{2},j_{2}}| > \kappa] = \\ 
\mathbb{P}[|\sum_{1 \leq n_{1} < n_{2} \leq n} d_{n_{1}}d_{n_{2}}x^{j_{1}}_{n_{1}}x^{j_{2}}_{n_{2}}2\sigma_{i_{1},i_{2}}(n_{1},n_{2})| > \kappa].
\end{split}
\end{align}

For $\{n_{1},n_{2}\}$ such that $n_{1} \neq n_{2}$ and $\sigma_{i_{1},i_{1}}(n_{1},n_{2}) \neq 0$
let us now consider random variables $Y_{n_{1},n_{2}}$ that are defined as follows
\begin{equation}
Y_{n_{1},n_{2}} = 2d_{n_{1}}d_{n_{2}}x^{j_{1}}_{n_{1}}x^{j_{2}}_{n_{2}}\sigma_{i_{1},i_{1}}(n_{1},n_{2}).
\end{equation}
From the definition of the chromatic number $\chi(i_{1},i_{1})$ we can deduce that the set of all this random
variables can be partitioned into at most $\chi(i_{1},i_{1})$ subsets such that random variables in each subset
are independent.
Let us denote these subsets as: $\mathcal{L}_{1},...,\mathcal{L}_{r}$, where $r \leq \chi(i_{1},i_{1})$.
Note that an event $\{|\sum_{1 \leq n_{1} < n_{2} \leq n}d_{n_{1}}d_{n_{2}}x^{j_{1}}_{n_{1}}x^{j_{2}}_{n_{2}}2\sigma_{i_{1},i_{1}}(n_{1},n_{2})| > \kappa\}$
is contained in the sum of the events: $\mathcal{E} = \mathcal{E}_{1} \cup ... \cup \mathcal{E}_{r}$, where each $\mathcal{E}_{j}$ is defined as follows:
\begin{equation}
\mathcal{E}_{j} = \{|\sum_{Y \in \mathcal{L}_{j}} Y| \geq \frac{\kappa}{\chi(i_{1},i_{1})}\}.
\end{equation}

Thus, from the union bound we get: 
\begin{equation}
\mathbb{P}[\mathcal{E}] \leq \sum_{i=1}^{\chi(i_{1},i_{1})} \mathbb{P}[\mathcal{E}_{i}].
\end{equation}
Now we can use Azuma's inequality to find an upper bound on $\mathcal{P}[\mathcal{E}_{i}]$
and we obtain:
\begin{equation}
\mathbb{P}[\mathcal{E}_{i}] \leq 2e^{-\frac{\frac{\kappa^{2}}{\chi^2(i_{1},i_{1})}}{2 \sum_{1 \leq n_{1} < 
n_{2} \leq n} (2\sigma_{i_{1},i_{1}}(n_{1},n_{2}))^{2}(x^{j_{1}}_{n_{1}})^{2}(x^{j_{2}}_{n_{2}})^{2}}}. 
\end{equation}

Now, if we assume that the vectors of the orthonormal basis $\mathcal{B}$ are $\log(n)$-balanced, then
by the union bound we obtain the following upper bound on the probability $\mathbb{P}[\mathcal{E}]$:

\begin{equation}
\mathbb{P}[\mathcal{E}] \leq 2\chi(i_{1},i_{1})e^{-\frac{\kappa^{2} n^{2}}{2\log^{4}(n)\chi^{2}(i_{1},i_{1})\sum_{1 \leq n_{1} < n_{2} \leq n}
(2\sigma_{i_{1},i_{1}}(n_{1},n_{2}))^{2}}}. 
\end{equation}

We can conclude, using the union bound again, that for a $\log(n)$-balanced basis $\mathcal{B}$ the probability that there exist $i_{1},j_{1},j_{2}$ such that:
$|s^{i_{1},j_{1}} \cdot s^{i_{1},j_{2}}| > \kappa$ is at most 
\begin{equation}
p_{1,bad}(\kappa) \leq 2\sum_{i=1}^{m} \chi(i,i)e^{-\frac{\kappa^{2}}{2\xi^{2}(i,i)}\frac{n^{2}}{\log^{4}(n)}}.
\end{equation}

Now let us find an upper bound on the expression $p_{2,bad}(\kappa) = \mathbb{P}[\exists_{i_{1},i_{2},j_{1},j_{2}, i_{1} \neq i_{2}} : |s^{i_{1},j_{1}} \cdot s^{i_{2},j_{2}}| > \kappa]$,
where $i_{1} \neq i_{2}$. 
We will assume that vectors of the basis $\mathcal{B}$
are $\log(n)$-balanced. 
Using the formula on $s^{i_{1},j_{1}} \cdot s^{i_{2},j_{2}}$ for $i_{1} \neq i_{2}$, we get:
\begin{align}
\begin{split}
 \mathbb{P}[|s^{i_{1},j_{1}} \cdot s^{i_{2},j_{2}}| > \kappa ] = \\ \mathbb{P}[|\sigma_{i_{1},i_{2}}(1,1)x^{j_{1}}_{1}x^{j_{2}}_{1}+...+\sigma_{i_{1},i_{2}}(n,n)x^{j_{1}}_{n}x^{j_{2}}_{n}\\
+ 2 \sum_{1 \leq n_{1} < n_{2} \leq n} d_{n_{1}}d_{n_{2}}x^{j_{1}}_{n_{1}}x^{j_{2}}_{n_{2}}\sigma_{i_{1},i_{2}}(n_{1},n_{2})|> \kappa]. 
\end{split}
\end{align}
Assume first that $P_{i}s$ are chosen deterministically. 
Note that by $\log(n)$-balanceness, we have: 
\begin{equation}
|\sum_{n_{1}=1}^{n} \sigma_{i_{1},i_{2}}(n_{1},n_{1})x^{j_{1}}_{1}x^{j_{2}}_{1}| \leq \frac{\log^{2}(n)}{n}\lambda(i_{1},i_{2}).
\end{equation}

Thus, by the triangle inequality, we have:
\begin{align}
\begin{split}
\mathbb{P}[|s^{i_{1},j_{1}} \cdot s^{i_{2},j_{2}}| > \kappa ] \\ \leq \mathbb{P}[|2 \sum_{1 \leq n_{1} < n_{2} \leq n} d_{n_{1}}d_{n_{2}}x^{j_{1}}_{n_{1}}x^{j_{2}}_{n_{2}}\sigma_{i_{1},i_{2}}(n_{1},n_{2})|
\geq \\ 
\kappa - \frac{\log^{2}(n)}{n}\lambda(i_{1},i_{2})]. 
\end{split}
\end{align}

Using the same analysis as before, we then obtain the following bound on $p_{bad}(\kappa, \theta)$:
\begin{equation}
 p_{2,bad}(\kappa) \leq 2\sum_{1 \leq i_{1} < i_{2} \leq m} \chi(i_{1},i_{2})e^{-\frac{(\kappa-\frac{\log^{2}(n)}{n}\lambda(i_{1},i_{2}))^{2}}{2\xi^{2}(i,i)}\frac{n^{2}}{\log^{4}(n)}}.
\end{equation}

We can conclude that in the setting where $P_{i}s$ are chosen deterministically, under our assumptions on $\lambda(i_{1},i_{2})$, for $\kappa > 0$
that does not depend on $n$ and
$n$ large enough the following is true. The probability that there exist two different vector $s^{i_{1},j_{1}}$, $s^{i_{2},j_{2}}$
such that $|s^{i_{1},j_{1}} \cdot s^{i_{2},j_{2}}| > \kappa$ satisfies:

\begin{equation}
p_{bad}(\kappa) \leq 2\sum_{1 \leq i_{1} \leq i_{2} \leq m} \chi(i_{1},i_{2})e^{-\frac{(\kappa-\frac{\log^{2}(n)}{n}\lambda(i_{1},i_{2}))^{2}}{2\xi^{2}(i,i)}\frac{n^{2}}{\log^{4}(n)}}.
\end{equation}

Now let us assume that $P_{i}s$ are chosen probabilistically. In that setting we also assume that columns of different indices are chosen
independently (this is the case for instance for the FastFood Transform).
Let us now denote: 
\begin{equation}
Y_{j} = \sigma_{i_{1},i_{2}}(j,j)x^{j_{1}}_{j}x^{j^{2}}_{j}
\end{equation}
for $j=1,...,n$.
Denote $Y = \sum_{i=1}^{n} Y_{1} + ... + Y_{n}$. Note that the condition on $\tilde{\lambda}(i_{1},i_{2})$
from the statement of the theorem implies that $\mathbb{E}[Y] = o_{n}(1)$.
From the condition regarding independence of columns of different indices we deduce that $Y_{i}s$ are independent.
Therefore we can apply Azuma's inequality and obtain the following bound on the expression: $\mathbb{P}[|Y - \mathbb{E}[Y]| > a]$:
\begin{align}
\begin{split}
\mathbb{P}[|Y-\mathbb{E}[Y]| > a] \leq 2e^{-\frac{a^{2}}{ 8\frac{\log^{4}(n)}{n^{2}} \sum_{j=1}^{n}(\sigma_{i_{1},i_{2}}^{max}(j,j))^{2}}}.
\end{split}
\end{align}

If we now take $a=\frac{1}{\log(n)}$ and under $\log(n)$-balanceness assumption, we obtain:
\begin{equation}
\mathbb{P}[|Y-\mathbb{E}[Y]| > a] \leq 2e^{-\frac{n^{2}}{ 8\log^{6}(n) \sum_{j=1}^{n}(\sigma_{i_{1},i_{2}}^{max}(j,j))^{2}}}. 
\end{equation}
Assume now that $|Y-\mathbb{E}[Y]| \leq \frac{1}{\log(n)}$. This happens with probability
at least $1-p_{wrong}$ with respect to the random choices of $P_{i}s$,
where $p_{wrong} = 2e^{-\frac{n^{2}}{ 8\log^{6}(n) \sum_{j=1}^{n}(\sigma_{i_{1},i_{2}}^{max}(j,j))^{2}}}$.
But then random variable $|Y|$ is of the order $o_{n}(1)$. 

Note that we have:
\begin{align}
\begin{split}
 \mathbb{P}[|s^{i_{1},j_{1}} \cdot s^{i_{2},j_{2}}| > \kappa ] = \\ \mathbb{P}[|Y
+ 2 \sum_{1 \leq n_{1} < n_{2} \leq n} d_{n_{1}}d_{n_{2}}x^{j_{1}}_{n_{1}}x^{j_{2}}_{n_{2}}\sigma_{i_{1},i_{2}}(n_{1},n_{2})|> \kappa]. 
\end{split}
\end{align}

Thus, using our bound on $Y$ for a fixed $\kappa$ and $n$ large enough we can repeat previous analysis and conclude that in the probabilistic setting of $P_{i}s$
the following is true:
\begin{equation}
p_{bad}(\kappa) \leq 2\sum_{1 \leq i_{1} \leq i_{2} \leq m} \chi(i_{1},i_{2})e^{-\frac{(\frac{\kappa}{2})^{2}}{2\xi^{2}(i,i)}\frac{n^{2}}{\log^{4}(n)}}.
\end{equation}

Thus we can conclude that in both the deterministic and probabilistic setting for $P_{i}s$ we get:
\begin{equation}
p_{bad}(\kappa) \leq 2\sum_{1 \leq i_{1} \leq i_{2} \leq m} \chi(i_{1},i_{2})e^{-\frac{\kappa^{2}}{8\xi^{2}(i,i)}\frac{n^{2}}{\log^{4}(n)}}.
\end{equation}

Now we will show that the squared lengths of vectors $s^{i,j}$ are well concentrated around their means and that these means are equal to $1$.
Let us remind that we have:
\begin{equation}
s^{i,j}_{l} = d_{1}p^{i}_{l,1}x^{j}_{1}+...+d_{n}p^{i}_{l,n}x^{j}_{n}.
\end{equation}
Thus we get:
\begin{align}
\begin{split}
\|s^{i,j}\|^{2}_{2} = \sum_{1 \leq n_{1} < n_{2} \leq n} d_{n_{1}}d_{n_{2}}x^{j_{1}}_{n_{1}}x^{j_{2}}_{n_{2}}2\sigma_{i,i}(n_{1},n_{2}) + \\
\sum_{n_{1}=1}^{n} (\sigma_{i,i}(n_{1},n_{1}))^{2}(x^{j}_{n_{1}})^{2}=\\
\sum_{1 \leq n_{1} < n_{2} \leq n} d_{n_{1}}d_{n_{2}}x^{j_{1}}_{n_{1}}x^{j_{2}}_{n_{2}}2\sigma_{i,i}(n_{1},n_{2}) + 1,
\end{split}
\end{align}
where the last inequality comes from the fact that each column of each $P_{i}$ has $l_{2}$-norm equal to $1$.

Since obviously $\mathbb{E}[d_{n_{1}}d_{n_{2}}x^{j_{1}}_{n_{1}}x^{j_{2}}_{n_{2}}2\sigma_{i,i}(n_{1},n_{2}) ] = 0$,
then indeed $\mathbb{E}[\|s^{i,j}\|^{2}_{2}] = 1$.
Let us find the upper bound on the following probability: 
$\mathbb{P}[|\|s^{i,j}\|^{2}_{2}-1| > \frac{1}{\log(n)}]$.
We have:
\begin{align}
\begin{split}
\mathbb{P}[|\|s^{i,j}\|^{2}_{2}-1| > \frac{1}{\log(n)}] = \\
\mathbb{P}[|d_{n_{1}}d_{n_{2}}x^{j_{1}}_{n_{1}}x^{j_{2}}_{n_{2}}2\sigma_{i,i}(n_{1},n_{2})| > \frac{1}{\log(n)}].
\end{split}
\end{align} 

We can again apply Azuma's inequality and the union bound as we did before and obtain:
\begin{align}
\begin{split}
\mathbb{P}[ \exists_{i,j} : |\|s^{i,j}\|^{2}_{2}-1| > \frac{1}{\log(n)}] 
\leq p_{s},
\end{split}
\end{align}
where $p_{s} = 4\sum_{i=1}^{m}\chi(i,i)e^{-\frac{1}{2\xi^{2}(i,i)\log^{2}(n)}\frac{n^{2}}{\log^{4}(n)}}$. 

We will assume now that all $s^{i,j}$ satisfy: 
$|\|s^{i,j}\|^{2}_{2}-1| \leq \frac{1}{\log(n)}$, in particular:
\begin{equation}
\sqrt{1-\frac{1}{\log(n)}} \leq \|s^{i,j}\|_{2} \leq \sqrt{1 + \frac{1}{\log(n)}}.
\end{equation}

Let us assume right now that the above inequality holds. Let $\{w^{i,j}\}$ be a set of vectors
obtained from $\{s^{i,j}\}$ by the Gram-Schmidt process. Without loss of generality we can assume that $\|w^{i,j}\|_{2} = \|s^{i,j}\|_{2}$. Note that the size of the set $\{s^{i,j}\}$
is in fact not $2m$, but $2r$ and in all practical application $r \ll m$.
Assume now that $|s^{i_{1},j_{1}} \cdot s^{i_{2},j_{2}}| \leq \kappa$ for any two different vectors $s^{i_{1},_{j_{1}}}, s^{i_{2},j_{2}}$ and some fixed $\kappa >0$.
Now, one can easily note that directly from the description of the Gram-Schmidt process that it leads to the set of vectors $\{w^{i,j}\}$ such that $\|s^{i,j}-w^{i,j}\|_{2} \leq \kappa \Gamma(2r)$, where $\Gamma$ is some constant that depends just on the size of the set $\{s^{i,j}\}$. 
Thus if we want $\rho$-orthogonality with $\rho = \frac{\epsilon}{\|g^{\mathcal{H}}\|_{2}}$,
where $g^{\mathcal{H}}$ stands for the random projection of a vector $g$ onto $2r$-dimensional linear space spanned by vectors from $\{s^{i,j}\}$,
then we want to have:
\begin{equation}
\frac{\epsilon}{\|g^{\mathcal{H}}\|_{2}} = \kappa \Gamma(2r).
\end{equation}
Thus we need to take:
\begin{equation}
\kappa = \frac{\epsilon}{\Gamma(2r)\|g^{\mathcal{H}}\|_{2}}.
\end{equation}
Note that $g^{\mathcal{H}}$ is a $2r$-dimensional gaussian vector.
Now let us take some $T>0$. By the union bound the probability that $g^{\mathcal{H}}$
has $l_{2}$ norm greater than $\sqrt{2r} \cdot \sqrt{T}$ is at most: $2r \mathbb{P}[|\hat{g}|^{2} > T]$, where $\hat{g}$ stands for a gaussian random variable taken from $\mathcal{N}(0,1)$.
Now we use the following inequality for a tail of the gaussian random variable:
\begin{equation}
\mathbb{P}[|\hat{g}| > x] \leq 2\frac{e^{-\frac{x^{2}}{2}}}{x\sqrt{2\pi}}.
\end{equation}
Thus we can conclude that the probability that $g^{\mathcal{H}}$ has $l_{2}$ norm larger than
$\sqrt{2r} \cdot \sqrt{T}$ is at most $p_{gauss}(T) \leq \frac{4r}{\sqrt{2\pi T}}$.
In such a case we need to take $\kappa$ of the form:
\begin{equation}
\kappa = \frac{\epsilon}{\Gamma(2r)\sqrt{2r} \sqrt{T}}.
\end{equation}

We are ready to finish the proof of Lemma \ref{technical_lemma}.
Take $\kappa = \frac{\epsilon}{\Gamma(2r)\sqrt{2r} \sqrt{T}}$.
Let us first take the setting where $P_{i}s$ are chosen deterministically.
Take an event $\mathcal{E}_{bad}$ which is the sum of the events which probabilisites are upper-bounded by $p_{gauss}(T)$, $1-p_{balanced}$, $p_{bad}(\kappa)$ and $p_{s}$.
By the union bound, the probability of that event is at most $p_{gauss} + (1-p_{balanced}) + p_{bad}(\kappa) + p_{s}$ which is upper-bounded by $p_{gen} + p_{struct}$ for $n$ large enough.
Note that if $\mathcal{E}_{bad}$ does not hold then $\rho$-orthogonality is satisfied.
Now let us take the probabilistic setting for choosing $P_{i}s$. We proceed similarly.
The only difference is that right now we need to assume that the event upper-bounded by $p_{wrong}$ does not hold (this one depends only on the random choices for setting up $P_{i}s$). Thus again we get the statement of the lemma.
That completes the proof of Lemma \ref{technical_lemma}.
\end{proof}
As mentioned above, the proof of Lemma \ref{technical_lemma} completes the proof of the theorem.
\end{proof}

Now we prove Theorem \ref{main_variance_theorem}.
\begin{proof}
Fix some $\vv{x}, \vv{z} \in \mathbb{R}^{n}$. Assume that a matrix $\vv{A}$ is used to compute the approximation of the kernel $k(\vv{x},\vv{z})$.
Matrix $\vv{A}$ is either a truly random Gaussian matrix as it is the case in the unstructured computation or a structured matrix produced according to the $\mathcal{P}$-model.
We assume that $\vv{A}$ has $k$ rows and consists of $\frac{k}{m}$ blocks stacked vertically. 
If $\vv{A}$ is produced via the $\mathcal{P}$-model then each block is a structured matrix $G^{i}_{struct}$.
The approximation of the kernel $\tilde{k}_{\mathcal{P}}(\vv{x},\vv{z})$ is of the form:
$\tilde{k}_{\vv{A}}(\vv{x},\vv{z}) = \frac{1}{k} \sum_{i=1}^{\frac{k}{m}} \sum_{j=1}^{m} [\phi(a^{i,j} \cdot \vv{x}, a^{i,j} \cdot \vv{y})]$,
where $a^{i,j}$ stands for the $j^{th}$ row of the $i^{th}$ block and
$\phi : \mathbb{R}^{2} \rightarrow \mathbb{R}$ is either of the form $\phi(a,b) = f(a)f(b)$, where $f$ is a ReLU/sign function or $\phi(a,b) = \cos(a)\cos(b) + \sin(a)\sin(b)$.
The latter formula for $\phi$ is valid if a kernel under consideration is Gaussian.
Let use denote the random variable: $\phi(a^{i,j} \cdot \vv{x}, a^{i,j} \cdot \vv{y})$ as $X_{i,j}$.
Then we have:
\begin{equation}
\tilde{k}_{\vv{A}}(\vv{x},\vv{z}) = \frac{1}{k}\sum_{i=1}^{\frac{k}{m}}\sum_{j=1}^{m}X_{i,j}.
\end{equation}

Thus we have:
\begin{align}
\begin{split}
\label{var_equation}
Var(\tilde{k}_{\vv{A}}(\vv{x},\vv{z})) = Var(\frac{1}{k}\sum_{i=1}^{\frac{k}{m}}\sum_{j=1}^{m}X_{i,j}) = \\ 
\frac{1}{k^{2}}Var(\sum_{i=1}^{\frac{k}{m}}\sum_{j=1}^{m}X_{i,j})=\frac{1}{k^{2}}[\sum_{i=1}^{\frac{k}{m}}\sum_{j=1}^{m} Var(X_{i,j}) + \\
\sum_{i,j_{1} \neq j_{2}}Cov(X_{i,j_{1}},X_{i,j_{2}})].
\end{split} 
\end{align}

The last inequality in Eqn.\ref{var_equation} is implied by the fact that different blocks of the structured matrix are computed
independently and thus covariance related to rows from different blocks is $0$.

Therefore we obtain:
\begin{align}
\begin{split}
Var(\tilde{k}_{\vv{A}}(\vv{x},\vv{z})) = \frac{1}{k^{2}}\sum_{i=1}^{\frac{k}{m}}\sum_{j=1}^{m}Var(X_{i,j}) + \\
\frac{1}{k^{2}}\sum_{i, j_{1} \neq j_{2}} 
(\mathbb{E}[X_{i,j_{1}},X_{i,j_{2}}]-\mathbb{E}[X_{i,j_{1}}]\mathbb{E}[X_{i,j_{2}}]).
\end{split}
\end{align}

Now note that the first expression on the RHS above is the same for both the structured and unstructured setting.
This is the case since one can note that $X_{i,j}$ has the same distribution in the unstructured and structured setting.
For the same reason the expression $\mathbb{E}[X_{i,j_{1}}]\mathbb{E}[X_{i,j_{2}}]$ is the same for the structured and
unstructured setting. Thus if $\vv{G}$ stands for the fully unstructured model and we denote 
$\tilde{k}_{\vv{A}}(\vv{x},\vv{z}) = \tilde{k}_{\mathcal{P}}(\vv{x},\vv{z})$ if $A$ is constructed according to the $\mathcal{P}$-model,
then we get:

\begin{align}
\begin{split}
|Var(\tilde{k}_{\vv{G}}(\vv{x},\vv{z}))-Var(\tilde{k}_{\mathcal{P}}(\vv{x},\vv{z}))| \leq \\
\frac{1}{k^{2}} \sum_{i,j_{1} \neq j_{2}}|\mathbb{E}[X^{\mathcal{P}}_{i,j_{1}}X^{\mathcal{P}}_{i,j_{2}}]-
\mathbb{E}[X^{\vv{G}}_{i,j_{1}}X^{\vv{G}}_{i,j_{2}}]|,
\end{split}
\end{align}
where $X^{\mathcal{P}}_{i,j}$ stands for the version of $X_{i,j}$ if $\vv{A}$ was costructed via the $\mathcal{P}$-model
and $X^{\vv{G}}_{i,j}$ stands for the fully unstructured one.

Therfore we have:
\begin{align}
\begin{split}
|Var(\tilde{k}_{\vv{G}}(\vv{x},\vv{z}))-Var(\tilde{k}_{\mathcal{P}}(\vv{x},\vv{z}))| \leq \\
\frac{1}{k^{2}}\cdot \frac{k}{m} \sum_{j_{1} \neq j_{2}}|\mathbb{E}[X^{\mathcal{P}}_{1,j_{1}}X^{\mathcal{P}}_{1,j_{2}}]-
\mathbb{E}[X^{\vv{G}}_{1,j_{1}}X^{\vv{G}}_{1,j_{2}}]|,
\end{split}
\end{align}
where the latter inequality is implied by the fact that different blocks are constructed independently.

Therefore we get:
\begin{equation}
|Var(\tilde{k}_{\vv{G}}(\vv{x},\vv{z}))-Var(\tilde{k}_{\mathcal{P}}(\vv{x},\vv{z}))| \leq \frac{1}{k^{2}} \cdot \frac{k}{m} {m \choose 2} \beta,
\end{equation}
where $\beta$ is an upper bound as in Theorem \ref{main_theorem} for $d=2$. Now we can proceed in the same way as in the proof of Theorem \ref{main_theorem_1}
and the proof is completed.
\end{proof}
Now we prove Theorem \ref{ldrm-intro}.
\begin{proof}
The fact that $\mu[\mathcal{P}] \leq \kappa$ comes directly from the definition of the coherence number and the sparse setting of semi-gaussian matrices. 
To see that, note that any given column $col$ of any matrix $\vv{P}_{i}$ in the related $\mathcal{P}$-model has a nonzero dot-product with at most $\kappa^{2}$ other columns
of any matrix $\vv{P}_{j}$. This in turn is implied by the fact that different columns are obtained by applying skew-circulant shifts blockwise, thus the number of 
columns from $\vv{P}_{j}$ that have nonzero dot product with $col$ is at most the product of the number of nonzero dimensions of $col$ and $\vv{P}_{j}$. This is clearly upper 
bounded by $\kappa^{2}$. This leads to the upper bound on the coherence $\mu[\mathcal{P}]$.

The new formula for $p_{wrong}$ is derived by a similar analysis to the one used to obtain the formula on $p_{wrong}$ in the proof of Theorem \ref{main_theorem_1}. 
This time random variables under analysis are not independent though, but using 
the same trick as the one we used in the proof of Theorem \ref{main_theorem_1}
to decouple dependent random variables in the sum to be estimated and applying Azuma's inequality (we omit details since the analysis is exactly the same as 
in the aforementioned proof), 
we obtain the following: $\mathbb{P}[|\vv{P}_{i,n_{1}}^{T} \vv{P}_{j,n_{1}}| >c] \leq e^{-\Omega (rc^{2})}$ for $i \neq j$ and any constant $c>0$. Taking the union bound
over all the pairs of columns and fixing $c=\frac{1}{\log^{2}(n)}$ and $r=3\log^{5}(n)$, we can conclude that with probability at 
least $1-o(\frac{1}{n})$ the absolute value of the expression $\lambda(i,j)$ 
from the proof of Theorem \ref{main_theorem_1} is of the order
$o(\frac{n}{\log^{2}(n)})$. That enables us to finish tha analysis in the same way as in the proof of Theorem \ref{main_theorem_1} and derive similar conclusions.

The bound regarding the chromatic number is implied by the observation that each coherence graph in the corresponding $\mathcal{P}$-model has degree at most $\kappa^{2}$.
That follows directly from the observation we used to prove the upper bound on $\mu[\mathcal{P}]$. But now we can use Lemma \ref{graph_lemma} and that completes the proof
of Theorem \ref{ldrm-intro}.
\end{proof}

Below we present the proof of Theorem \ref{ldrm-extra}.

\begin{proof}
Fix two columns $\vv{P}_{i,n_{1}}$ and $\vv{P}_{j,n_{2}}$ and consider the expression $\vv{P}_{i,n_{1}}^{T}\vv{P}_{j,n_{2}}$. 
We have already mentioned in the previous proof the right approach to finding strong upper bound on $|\vv{P}_{i,n_{1}}^{T}\vv{P}_{j,n_{2}}|$.
We first note that $\vv{P}_{i,n_{1}}^{T}\vv{P}_{j,n_{2}}$ can be written as a sum $w_{1} + ... + w_{nr}$, where $w_{i}s$ are not necessarily independent but can be
partitioned into at most three sets such that wariables in each of these sets are independent. This is true since $G^{i}_{struct}$ is produced by skew-circulant shifts
and the corresponding coherence graphs has verrtices of degree at most $2$. 
Note also that each $w_{k}$ satisfies: $|w_{k}| \leq \frac{1}{\alpha r}$.
In each of the sum we get rid of these $w_{i}s$ that are equal to $0$.
Then, by applying Azuma's inequality independently on each of these subsets
and taking union bound over these subsets, we conclude that for any $a>0$:
\begin{equation}
|\vv{P}_{i,n_{1}}^{T}\vv{P}_{j,n_{2}} > a| \leq 3e^{-\frac{a^{2}\alpha r}{O(1)}} 
\end{equation}
Now we can take the union bound over all pairs of columns and notice that for every columcn $col$ in $\vv{P}_{i}$ and any $\vv{P}_{j}$ there exists at most $\kappa$
columns in $\vv{P}_{j}$ that have nonzero dot product with $col$. We can then take $a = \frac{\tau}{\kappa}$ and the proof is completed.
\end{proof}

Let us now switch to dense semi-gaussian matrices.
The following is true.

\begin{theorem}
\label{random_theorem}
Consider the setting as in Theorem \ref{main_theorem_1}.
Assume that entries of any fixed column of $P_{i}$ are chosen independently at random. 
Assume also that for any $1 \leq i \leq j \leq m$ and any fixed column $col$ of $P_{i}$ 
each column of $P_{j}$ is a downward shift of $col$ by $b$ entries (possibly with signs of dimensions swapped) and that $b=0$ for $O(1)$ columns in $P_{j}$.
Then for and $T >0$ and $n$ large enough the following holds:
\begin{align}
\begin{split}
|\mathbb{E}[\tilde{k}_{\mathcal{P}}^{d}(\textbf{x},\textbf{z})]-\mathbb{E}[\tilde{k}_{\vv{G}}^{d}(\textbf{x},\textbf{z})]|  \leq 
O(\Delta),
\end{split}
\end{align}
where 
$\Delta = p_{gen}(T) + p_{struct}(T) + d\epsilon + e^{-n^{\frac{1}{3}}}$
and
$$\epsilon = \frac{\log^{3}(n)}{n}\left(n^{\frac{2}{3}} + \max_{1 \leq i \leq j \leq m} |\sum_{1 \leq n_{1} < n_{2} \leq n} \vv{P}_{i,n_{1}}^{T}\vv{P}_{j,n_{2}}|\right).$$
As a corollary:
\begin{align}
\begin{split} 
|Var(\tilde{k}_{\mathcal{P}}(\textbf{x},\textbf{z})) -Var(\tilde{k}_{\vv{G}}(\textbf{x},\textbf{z}))|= 
O(\frac{m-1}{2k}\Delta).
\end{split}
\end{align}
\end{theorem}

\begin{proof}
The proof of this result follows along the lines of the proof of Theorem \ref{main_theorem_1} and Theorem \ref{main_variance_theorem}.
Take the formulas for $s^{i_{1},j_{1}} \cdot s^{i_{2},j_{2}}$ derived in the proof of Theorem \ref{main_theorem}. Note that we want to have:
$|s^{i_{1},j_{1}} \cdot s^{i_{2},j_{2}}| \leq \frac{\epsilon}{\Gamma(2d)\|g^{\mathcal{H}}\|_{2}}$, where $\Gamma$ is a constant that depends only on
the degree $d$. Each $s^{i_{1},j_{1}} \cdot s^{i_{2},j_{2}}$ is a sum of random variables that can be decoupled into $O(1)$ subsums such that variables
in each subsum are independent (here we use exactly the same trick as in the proof of Theorem \ref{main_theorem_1}). In each subsum we apply Azuma's inequality.
Straightforward computations lead to the conclusion that if one sets up $\epsilon$ as in the statement of Theorem \ref{random_theorem} then the probability that 
there exist different $s^{i_{1},j_{1}}$, $s^{i_{2},j_{2}}$ such that $|s^{i_{1},j_{1}} \cdot s^{i_{2},j_{2}}| > \frac{\epsilon}{\Gamma(2d)\|g^{\mathcal{H}}\|_{2}}$
is of the order $e^{-n^{\frac{1}{3}}}$ for $n$ large enough. That is the extra term in the formula for $\Delta$ that was not present in the staement of Theorem \ref{main_theorem_1}.
The variance results follows immediately by exactly the same analysis as in the proof of Theorem \ref{main_variance_theorem}.
\end{proof}

Note that introduced dense semi-gaussian matrices trivially satisfy conditions of Theorem \ref{random_theorem} 
(look for the description of matrices $\vv{P}_{i}$ from Subsection: \ref{subsec:semi}). 
The role of rank is similar as in the sparse setting, i.e. larger values of $r$ lead to sharper concentration results.
Theorem \ref{random_theorem} can be applied to classes of matrices for which $|\sum_{1 \leq n_{1} < n_{2} \leq n} \vv{P}_{i,n_{1}}^{T}\vv{P}_{j,n_{2}}|$
is small and random dense semi-gaussian matrices satisfy this condition with high probability.